\pdfoutput=1
\documentclass[11pt]{article}
\usepackage[a4paper,left=2.45cm,right=2.45cm,top=2.6cm,bottom=2.6cm,headheight=14pt]{geometry}
\usepackage{charter}
\usepackage{amsmath,amssymb,amsthm}
\usepackage{graphicx}
\usepackage{booktabs}
\usepackage{microtype}
\usepackage[dvipsnames]{xcolor}
\usepackage{enumitem}
\usepackage{fancyhdr}
\usepackage{titlesec}
\usepackage{caption}
\usepackage[round]{natbib}
\usepackage{placeins}
\usepackage[colorlinks=true,linkcolor=blue!65!black,citecolor=blue!65!black,urlcolor=blue!65!black]{hyperref}

\titleformat{\section}{\Large\bfseries}{\thesection.}{0.45em}{}
\titleformat{\subsection}{\large\bfseries}{\thesubsection.}{0.45em}{}
\titlespacing*{\section}{0pt}{14pt}{4pt}\titlespacing*{\subsection}{0pt}{10pt}{2pt}
\newcommand{\para}[1]{\textbf{#1}\hspace{0.6em}}
\DeclareCaptionLabelSeparator{bar}{\ |\ }
\setlist{itemsep=2pt,topsep=2pt}
\newcommand{\runtitle}{When Does a Diffusion Model Decide What to Draw?}
\fancypagestyle{main}{\fancyhf{}\fancyhead[C]{\small\runtitle}\fancyfoot[R]{\small\thepage}}
\fancypagestyle{first}{\fancyhf{}\fancyhead[L]{\large\bfseries Preprint}\fancyhead[R]{\small\textit{2026-10-6}}\fancyfoot[L]{\small\textit{Corresponding author: Snigdha Chandan Khilar, \href{mailto:snkhilar@gmail.com}{snkhilar@gmail.com}}}}
\newcommand{\E}{\mathbb{E}}
\newcommand{\sbif}{\sigma_{\rm bif}}
\newcommand{\tbif}{\tau_{\rm bif}}
\newtheorem{proposition}{Proposition}
\begin{document}
\thispagestyle{first}

\vspace*{2mm}
{\LARGE\bfseries When Does a Diffusion Model Decide What to Draw?\par}
\vspace{10pt}
{\normalsize\bfseries Snigdha Chandan Khilar\par}
{\small Independent Researcher\quad\textperiodcentered\quad\href{mailto:snkhilar@gmail.com}{snkhilar@gmail.com}\par}
\vspace{12pt}

{\bfseries
A diffusion model starts from pure noise and removes it step by step. Somewhere along the way it stops being able to become ``anything'' and becomes committed to, say, a horse rather than a truck. We measure when this happens, and what a trained model gets wrong about it, on CIFAR-10. The most direct measurement is to freeze a half-finished image, restart the generation from that point many times, and count how often each class comes out. We call this probability the committor. Measured this way, the model settles coarse questions (vehicle or animal?) at roughly twice the noise level of fine ones (which animal?). A much cheaper measurement, the noise level at which a classifier's opinion about two classes splits into two distinct groups, gets the order of these decisions right (rank correlation 0.73--0.88) but not their exact timing. We then compare pretrained models with their training data, using a comparison that is exact in principle. For NVIDIA's EDM model the result is simple: along the class axis, the model commits exactly as its data do, except that it generates some classes too often and others too rarely. If we give the data the model's class proportions, the data reproduce the model's behaviour pair by pair, as closely as two halves of the data reproduce each other. If we give the model the data's proportions, by sampling a class-conditional model with uniform labels or by gently steering the sampler, the model reproduces the data's behaviour, with no loss in image quality. Part of the class imbalance comes from the sampler rather than the network: a coarse 256-step sampler roughly doubles it. An older DDPM model shows the same class-proportion effect plus a second error, committing about 20\% too late, which our planted-error tests show the method is able to detect.
\par}
\vspace{4pt}
{\small Code: \url{https://github.com/nssprogrammer/diffusion-commit}\par}

\section{Introduction}

Watch a diffusion model generate an image and the picture comes into focus in a recognisable order. First there is a blur of colour, then a rough layout, then an object, and finally fine texture. Somewhere in the middle of this process the model \emph{commits}. Before that moment, the half-finished image could still turn into several different things. After it, the outcome is settled: it will be a horse, and no amount of further sampling noise will turn it into a truck.

Knowing when this happens is useful. It tells us where in the noise schedule guidance should act \citep{kynkaanniemi2024applying}, where the class balance of generated data is decided, and which errors in a learned model actually change what gets drawn. It also connects to a line of theory. For simple data such as a mixture of two Gaussian blobs, the moment of commitment is a sharp event that physicists would call a phase transition \citep{raya2023spontaneous,biroli2023generative,biroli2024dynamical,li2024critical,sclocchi2025phase}.

Measuring commitment in a real, trained model raises two practical questions that the theory hides.

\emph{Whose commitment do we mean?} One option is to look at the whole population of half-finished images and ask when they visibly split into two groups, for example ``cat-like'' and ``dog-like''. That is cheap, but what you see depends on how you choose to measure ``cat-likeness''. The other option is to look at one image at a time and ask how certain its eventual outcome already is. That is exact, but expensive. We use both and check one against the other.

\emph{Compared with what?} A model can only commit ``wrongly'' relative to its training data. So we need a fair way to compare the model's behaviour with the data's, and evidence that this comparison would actually notice an error if one existed.

We study CIFAR-10 \citep{krizhevsky2009learning}, with NVIDIA's EDM models \citep{karras2022elucidating} as the main subject and the original DDPM model \citep{ho2020denoising} as a second, older model family. Our findings are:
\begin{enumerate}[leftmargin=*]
\item \textbf{Coarse decisions come first.} An image settles ``vehicle or animal?'' at about twice the noise level at which it settles ``which animal?'' (1.75--2.5 times, depending on how strictly ``settled'' is defined). The cheap population-level measurement puts class pairs in the right order (rank correlation 0.73--0.88), but its absolute timing depends on an arbitrary choice of scale, so we use it only for ordering (Section~\ref{sec:cifar}).
\item \textbf{The comparison is fair, and it can see errors.} On a toy problem where the right answer is known, it recovers that answer exactly. When we deliberately plant small errors in the data, it detects them (Sections~\ref{sec:compare} and~\ref{sec:power}).
\item \textbf{EDM commits like its data, except for class proportions.} Reweighting the data to match how often the model generates each class reproduces the model's behaviour pair by pair, as closely as two halves of the data agree with each other. A class-conditional EDM model sampled with equal numbers of each class matches the data with no reweighting at all (Section~\ref{sec:edm}).
\item \textbf{Fixing the proportions fixes the commitment.} Gently steering the sampler towards the right class proportions, with a classifier and no tuning, makes the model's commitment match the data, and FID does not get worse (Section~\ref{sec:correct}).
\item \textbf{Part of ``the model's'' imbalance belongs to the sampler.} A 256-step stochastic sampler roughly doubles the class imbalance compared with the same network sampled more carefully (Section~\ref{sec:sampler}).
\item \textbf{A weaker model makes an additional error.} DDPM shows the same class-proportion effect, but it also commits about 20\% too late, an error that is not about proportions at all (Section~\ref{sec:ddpm}).
\end{enumerate}

\begin{table}[t]
\centering
\caption{Key terms, in the order they appear.}
\label{tab:glossary}
\small
\begin{tabular}{p{3.6cm}p{11.4cm}}
\toprule
term & meaning \\
\midrule
noise level $\sigma$ & How much Gaussian noise is mixed into an image. Generation starts at a large $\sigma$ (pure noise) and ends near $\sigma=0$ (a clean image). Higher noise means \emph{earlier} in generation. \\
commitment & The point in generation after which the final class of an image is effectively fixed. \\
committor & For a half-finished image, the probability that generation continued from it ends in a given class. Estimated by restarting generation many times from the same point. \\
decision noise level & The noise level below which a trajectory's outcome is settled, i.e. its committor stays close to 0 or 1. \\
order parameter $m$ & A single number that says how ``class $a$-like versus class $b$-like'' an image is. We use a classifier's log-odds. \\
density transition $\sbif$ & The noise level at which the distribution of $m$ across many images splits from one hump into two. \\
fold, pitchfork & The two shapes such a split can take: lopsided (one class first, the other appears later) or symmetric (both appear together). \\
asymmetry $h$ & How lopsided the split is. $h=0$ is perfectly symmetric; $|h|=0.08$ is the lopsidedness of a 54/46 class split. \\
class mass & How often each class is generated. A model with the wrong class mass makes some classes too often. \\
mass-matching & Reweighting the training data so its class proportions equal the model's, then comparing. \\
rebalancing & Reweighting the model's samples so that two classes carry equal weight, then comparing. \\
class-block interval & A confidence interval computed by resampling whole classes, since each class appears in nine of the 45 class pairs. \\
\bottomrule
\end{tabular}
\end{table}

\begin{figure}[t]
\centering
\includegraphics[width=\textwidth]{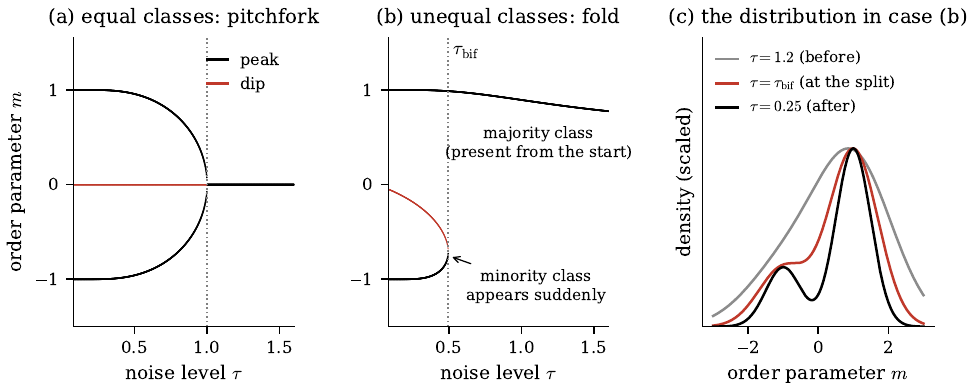}
\caption{Two ways a population of images can split into two classes. We follow the distribution of an order parameter $m$ (``how $a$-like versus $b$-like'') for a mixture of two Gaussian classes centred at $m=\pm1$, as noise is removed. The horizontal axis is $\tau$, the noise variance in units of the class separation, so generation moves from right to left. Black lines are the peaks of the distribution, red lines its dips. (a) With equal numbers of each class, the single peak splits symmetrically into two: a pitchfork. (b) With 75\% of one class, the single peak simply becomes the majority class, and the minority class appears later, suddenly, as a new peak with a dip next to it: a fold. (c) The distribution itself in case (b), before, at and after the split.}
\label{fig:fold}
\end{figure}

\section{Two Ways to Describe Commitment}
\label{sec:background}

\para{Diffusion models in one paragraph.}
Take a clean image $x_0$ and add Gaussian noise of standard deviation $\sigma$: $x_\sigma=x_0+\sigma\varepsilon$. A diffusion model is a network $D(x,\sigma)$ trained to guess the clean image from the noisy one, that is, $D(x,\sigma)\approx\E[x_0\mid x_\sigma=x]$. From this guess one can compute the \emph{score}, the direction in which noisy images become more likely: $\nabla\log p_\sigma(x)=(D(x,\sigma)-x)/\sigma^2$. To generate an image, start from pure noise at a large $\sigma$ and repeatedly move along the score while lowering $\sigma$. This can be done deterministically, by following an ordinary differential equation (the probability-flow ODE), or with some fresh noise injected at every step (the reverse-time SDE) \citep{song2021score,karras2022elucidating}. With a perfect score, both produce samples from the training distribution. They differ in a way that will matter in Section~\ref{sec:compare}.

\para{The committor: one image at a time.}
Pause the stochastic sampler at some noise level $\sigma$ and take the half-finished image $x$. Now restart generation from $x$ many times, each time with different random noise, and look at what comes out. If 11 of 12 restarts end as a horse, the image is nearly committed to ``horse''. If the outcomes are spread over several classes, it is not yet committed. The probability $q_c(x,\sigma)$ that a restart ends in class $c$ is the \emph{committor}, a term from the study of rare transitions in chemistry and physics \citep{e2010transition}. It describes a single image, needs no choice of measuring stick, and can be estimated simply by counting outcomes. Its drawback is cost, since every estimate needs many restarts. One fact makes it especially useful: for a perfect model, the committor of a noisy image is exactly the posterior probability of each class given that noisy image, $P(c\mid x_\sigma=x)$, under the training data (Appendix~\ref{app:joint}). A classifier trained on noisy data estimates that posterior, so we can check whether a model's committor matches what the data say it should be.

\para{The density transition: the whole population at once.}
The cheaper description looks at many images together. For a pair of classes, say cat ($a$) and dog ($b$), pick a single number $m$ that says how cat-like versus dog-like an image is, scaled so that clean cats sit near $+1$ and clean dogs near $-1$. We call $m$ an \emph{order parameter}. At high noise, every image looks equally ambiguous, and the distribution of $m$ over many noisy images has a single hump in the middle. At low noise, the cats and dogs separate, and the distribution has two humps. The noise level at which the second hump first appears is the \emph{density transition} $\sbif$.

\para{Pitchforks and folds.}
How does one hump become two? Figure~\ref{fig:fold} shows the two possibilities. If the two classes are perfect mirror images (same size, same spread), the hump splits evenly into two, a \emph{pitchfork}. If they are not, which is always the case for real data, the hump drifts towards one class and simply becomes it, and the other class appears later, suddenly, as a new hump. This is a \emph{fold}. Folds are the generic case \citep{stancevic2026information,achilli2026theory,marta2026speciation}. Both shapes are captured by one simple formula. Write $f(m)$ for the slope of the log-distribution of $m$; it is positive to the left of a hump and negative to its right. Near the split,
\begin{equation}
f(m)\approx h+a\,m-b\,m^3,\qquad b>0,
\label{eq:cusp}
\end{equation}
which is the textbook ``cusp'' normal form \citep{golubitsky1985singularities,strogatz2018nonlinear}. Each coefficient has a plain meaning. $a$ measures how strongly the noise is pushing the classes apart, and it grows as noise is removed. $b$ keeps the two humps at a finite distance. $h$ measures how lopsided the split is. With $h=0$ the split is a pitchfork. With $h\ne0$ it is a fold, and the second hump appears exactly when
\begin{equation}
4a^3=27\,b\,h^2.
\label{eq:fold}
\end{equation}
If $h>0$ the class at $+1$ is the one present from the start (Appendix~\ref{app:cusp}).

\para{Why class proportions matter, and why that alone proves little.}
Suppose one class is generated more often than the other. In Figure~\ref{fig:fold}b, that is exactly what makes the split lopsided: the majority class wins the single hump. The effect can be made precise. If $m$ is a log-likelihood ratio, i.e.\ a well-calibrated classifier's log-odds, then changing the class proportions from $w_a{:}w_b$ shifts $h$ by half the change in $\log(w_a/w_b)$ (Appendix~\ref{app:tanh}). So $|h|=0.08$ is the lopsidedness that a 54/46 class split would produce. Equation~\eqref{eq:fold} then adds a second consequence. A more lopsided split (larger $|h|$) needs a stronger push (larger $a$) before the second hump can appear, so it happens later in generation, at a lower noise level:
\begin{equation}
\Delta\log\sbif\;\approx\;-\frac{3\,(b/4)^{1/3}}{s}\;\Delta|h|^{2/3},\qquad s=-\frac{\mathrm da}{\mathrm d\log\sigma}.
\label{eq:shift}
\end{equation}
This is a useful sanity check, but it carries a warning. Any model that generates the wrong class proportions will show linked errors in lopsidedness and in timing, \emph{whatever else is wrong with it}. Seeing those linked errors therefore cannot prove that proportions are the \emph{only} problem. The decisive test, in Section~\ref{sec:edm}, is to give the data the model's proportions and check whether everything else then matches.

\section{How We Measure}
\label{sec:instruments}

\subsection{Committors, by restarting generation}
\label{sec:branching}

We generate 320 images with the unconditional EDM model, using the stochastic sampler. We pause each one at 12 noise levels between $7.9$ and $0.1$, spaced by a factor of about $1.5$. From every paused image we restart generation 12 times, 46{,}080 restarts in all, and identify the class of each finished image with a separate, highly accurate classifier: a Vision Transformer fine-tuned on CIFAR-10 \citep{dosovitskiy2021image}, with 98.4\% test accuracy. We say a trajectory has \emph{decided} its superclass (vehicle or animal) or its class at a noise level if at least 11 of the 12 restarts from that level, and from every lower level, agree with its own final outcome.

For a pair of classes we pool information across trajectories. Among all trajectories that end as $a$ or $b$, we look at the restarts that land in $a$ or $b$ and ask what fraction land on the \emph{other} class of the pair. We call this the cross-over rate. A pair is decided below the noise level from which the cross-over rate stays under a threshold. We report three thresholds ($0.05$, $0.1$ and $0.2$), because the exact noise level depends on how strict one is. We tested this procedure on a toy problem with six Gaussian classes, where committors are known exactly (Appendix~\ref{app:testbed}). The pooled estimate matches its exact target with no bias from using only 12 restarts. Different reasonable definitions of ``decided'' differ by about one step of our noise grid.

\subsection{Density transitions, with a classifier as the measuring stick}
\label{sec:density}

To measure how cat-like versus dog-like a noisy image is, we need an order parameter. For simple data like handwritten digits, a straight line through pixel space works. For CIFAR-10 it does not: along the best straight lines, the classes overlap so much that the distribution never shows two humps (Appendix~\ref{app:linear}). So we use a classifier. We train a ResNet that takes a noisy image and its noise level and outputs class probabilities (85\% accurate on nearly clean test images, 64\% at $\sigma=1$). For a pair of classes, the order parameter is the classifier's log-odds $z=\ell_a-\ell_b$, rescaled at each noise level so that the average $a$ sits at $+1$ and the average $b$ at $-1$. We estimate the distribution of $z$ across noisy images with a smooth histogram (a kernel density estimate), scan 56 noise levels from high to low until a second hump appears, and fit Equation~\eqref{eq:cusp} there to read off $h$.

Two details keep the numbers honest. First, statistics that pool many pairs need special care: each class appears in nine of the 45 pairs, so the pairs are not independent. Our confidence intervals therefore resample whole classes (\emph{class-block intervals}, 1{,}000 resamples), and per-pair estimates are medians of 100 bootstrap resamples. Second, an estimated $h$ is never exactly zero, so we need to know how large it can be by chance. We built pairs of image sets that are \emph{exactly} symmetric by construction: a class and its vertical flip, or a class and a copy with red and blue channels swapped. We passed them through a classifier trained on them. They give $|h|$ up to $0.047$, which serves as our noise floor (Appendix~\ref{app:coord}).

The classifier is a measuring stick, and the scale of the stick matters. If we replaced the log-odds $z$ by a squashed version such as $\tanh(z/2)$, every density transition would move, by a median factor of $4.5$. What does \emph{not} depend on the stick is the ordering of class pairs, which is stable across two very different classifier architectures ($r=0.96$) and a recalibration ($r=0.97$). Nor does any comparison between a model and its data \emph{made with the same stick}, because two identical distributions look identical under any stick. We rely only on these robust quantities.

\subsection{Comparing a model with its data}
\label{sec:compare}

To compare a model with its training data, we record the classifier's log-odds along the model's generation trajectories and compare them with the log-odds of real images that have had noise added to them. This is only fair if a perfect model would give exactly the same result as the data. For the stochastic sampler it does:

\begin{proposition}
\label{prop:joint}
Let $w_c$ be any classifier's probability for class $c$ on nearly clean images (noise level $\sigma_{\min}$), and let $\sigma>\sigma_{\min}$. With a perfect score, the distribution of $\phi(x_\sigma)$ along reverse-SDE trajectories, with each trajectory weighted by $w_a+w_b$ at its end point, equals the distribution of $\phi(x_\sigma)$ obtained by adding noise to real images, with each image weighted by $w_a+w_b$ at $\sigma_{\min}$.
\end{proposition}

The intuition is that the stochastic sampler, run with a perfect score, is the noising process played backwards in time \citep{anderson1982reverse,haussmann1986time}. Any statement about how an image at one noise level relates to the same image at another noise level is therefore identical on both sides (Appendix~\ref{app:joint}). The deterministic ODE sampler does not have this property: it produces the right distribution at each noise level, but it links noise levels together differently. We confirmed both facts on a toy problem with three Gaussian classes and an exactly known score. The SDE comparison returns a ratio of $1.00$ ($[0.94,1.07]$), meaning no difference, while the ODE returns $1.36$ ($[1.20,1.54]$). Model and data are always compared at the same sample size, because the smoothing in our histograms depends on it.

Two kinds of reweighting appear below. \emph{Rebalancing} rescales the model's samples so that the two classes of a pair carry equal weight. \emph{Mass-matching} goes the other way: it reweights the \emph{data} so that its class proportions equal the model's. The second is the cleaner test, because it leaves the model untouched.

\subsection{A check of the measuring pipeline}
\label{sec:check}

If the pipeline works, changing class proportions by a known amount should change $h$ by half the change in log class ratio, regardless of any model. We checked this on the data alone, reweighting the test images by random known class factors. The slope was $1.00$ ($r=0.98$, over the 38 pairs whose split was detected both before and after reweighting). Imposing a known imbalance on the model through guidance gave a slope of $1.05$. These numbers confirm that the classifier behaves like a calibrated likelihood ratio and that the cubic fit is accurate. They say nothing about whether any model is right.

\subsection{Can the comparison see errors?}
\label{sec:power}

\begin{figure}[t]
\centering
\includegraphics[width=0.9\textwidth]{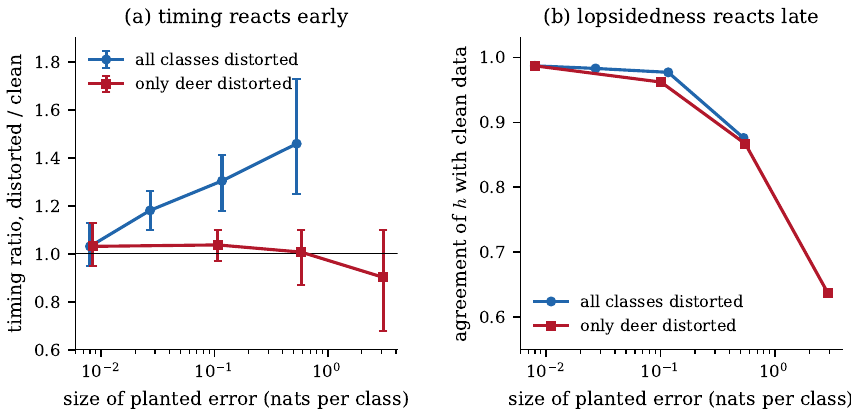}
\caption{Can the comparison detect errors? We distort the real data on purpose, by reweighting images within each class towards more typical-looking examples. Class proportions stay unchanged. We then compare distorted with clean data exactly as we compare a model with data. The horizontal axis is the size of the distortion (KL divergence, nats per class). (a) Ratio of split noise levels, averaged over pairs, with class-block intervals: the timing of the split reacts to small distortions. (b) Agreement of the lopsidedness $h$: it barely reacts until the distortion is large. Blue: all classes distorted; red: only deer.}
\label{fig:plant}
\end{figure}

A test that always says ``looks fine'' is worthless. To find out what ours can and cannot see, we planted errors on purpose. Within each class, we reweighted real images towards the more typical-looking ones, which makes classes tighter, without changing how many images each class has. Then we compared the distorted data with clean data, exactly as we compare a model with data (Figure~\ref{fig:plant}). The two quantities we measure turn out to be very different in sensitivity.
\begin{itemize}[leftmargin=*]
\item \emph{The timing of the split is sensitive.} A distortion of only $0.027$ nats per class moves split noise levels by a factor of $1.18$ ($[1.10,1.26]$).
\item \emph{The lopsidedness is not.} Its agreement with clean data stays at $0.98$ for distortions up to $0.12$ nats, and drops to $0.88$ only at $0.53$ nats.
\item \emph{A distortion in one class shows up mainly in that class's pairs.} Distorting deer alone barely moves the average timing, but it increases the spread across pairs and changes the lopsidedness of several deer pairs.
\end{itemize}
One limitation is built in. The classifier only looks along the ``which class is it?'' direction, so errors in other aspects of images, such as texture or style within a class, are invisible to this test.

\section{Experimental Setup}
\label{sec:setup}

This section collects every detail needed to reproduce our experiments; the code is available at \url{https://github.com/nssprogrammer/diffusion-commit}. Further details, and the checks behind several choices, are in the appendices.

\para{Data.}
All experiments use CIFAR-10 \citep{krizhevsky2009learning}: 50{,}000 training and 10{,}000 test images of $32\times32$ pixels in ten classes, with pixel values scaled to $[-1,1]$. Noise levels $\sigma$ are on this scale. The \emph{data reference} for every comparison is the test set, with four independent noise draws per image (40{,}000 noisy images per noise level), so the classifier is never evaluated on images it was trained on.

\para{Classifiers.}
We use four classifiers, each for a different job.
\begin{itemize}[leftmargin=*]
\item \emph{The census classifier} defines the order parameter. It is a ResNet with an input convolution of 64 channels and five residual blocks (64, 128, 128, 256 and 256 channels, two with stride 2), with BatchNorm and SiLU activations. It takes the noise level through a sinusoidal embedding of $\frac14\log\sigma$, followed by a small MLP and a scale-and-shift (FiLM) in every block \citep{perez2018film}. It is trained on the training set for 25 epochs with AdamW (learning rate $2\times10^{-3}$ on a one-cycle schedule, weight decay $5\times10^{-4}$, batch size 256), with horizontal flips and noise levels drawn log-uniformly between $0.05$ and $10$. Its test accuracy is 85.2\%, 81.8\%, 76.0\%, 64.2\%, 48.6\% and 34.4\% at $\sigma=0.05$, $0.25$, $0.5$, $1$, $2$ and $4$.
\item \emph{A second architecture} checks that results do not depend on the census classifier. It is a plain VGG-style network with GroupNorm, which takes the noise level by rescaling its input by $1/\sqrt{1+\sigma^2}$ and appending $\frac14\log\sigma$ as an extra input channel.
\item \emph{The null classifier} has the census architecture with 30 outputs: the ten classes, their vertical flips, and their red--blue channel swaps. It is used only to build the noise floor of Section~\ref{sec:density}.
\item \emph{The labelling classifier} labels finished images: a ViT-Base fine-tuned on CIFAR-10 at $224\times224$ resolution \citep{dosovitskiy2021image}, with 98.4\% test accuracy. It is independent of everything else in the pipeline.
\end{itemize}

\para{Diffusion models.}
NVIDIA's pretrained EDM checkpoints for CIFAR-10 \citep[unconditional and class-conditional, VP preconditioning, 56M parameters;][]{karras2022elucidating}, and the original DDPM checkpoint \citep{ho2020denoising}. DDPM is wrapped as described in Section~\ref{sec:ddpm}.

\para{Samplers.}
Unless stated otherwise, we sample with the reverse-time SDE using 256 Euler--Maruyama steps on the EDM noise schedule (exponent $\rho=7$), from $\sigma=80$ down to $\sigma_{\min}=0.05$, in half precision, and finish with one denoising step. For comparison we also use EDM's standard deterministic Heun sampler (18 steps, 35 network evaluations, down to $\sigma=0.002$) and the stochastic sampler with 512 and 1{,}024 steps, all in full precision.

\para{Recording.}
Along every generated trajectory we record the census classifier's logits at the 68 schedule steps between $\sigma=7.95$ and $\sigma=0.086$. At the end, the classifier's class probabilities at $\sigma_{\min}$ give each trajectory a weight for each class. Data-reference images are treated identically: noised first to $\sigma_{\min}$ (for their class weights) and then onwards (for their logits).

\para{Density transitions.}
For each class pair, the log-odds are standardised at each noise level using the data reference, so that the two class means sit at $\pm1$, and the same standardisation is applied to the model. The distribution is estimated with a Gaussian kernel density (bandwidth 1.5 times Silverman's rule) on a grid of 4{,}001 points over $[-4,4]$. Peaks closer than $0.03$, or separated by less than $0.05$ nats of log-density, are merged. The census scans 56 noise levels, spaced logarithmically from 20 down to $0.08$. Model comparisons scan the 68 recorded levels. The transition $\sbif$ is the first level with two peaks after at least one level with a single peak. The asymmetry $h$ comes from a density-weighted fit of Equation~\eqref{eq:cusp} over $|m|<1$ at that level.

\para{Committors.}
We use 320 trajectories of the unconditional EDM model, paused at 12 noise levels between $7.95$ and $0.099$, with 12 restarts from each paused image (46{,}080 restarts). Finished images are labelled by the ViT. Decisions use the 11-of-12 rule for superclass and class, and pooled cross-over rates with thresholds $0.05$, $0.1$ and $0.2$ for pairs. A level counts only if at least 20 restarts landed in the pair.

\begin{table}[t]
\centering
\caption{All sampling runs. ``Images'' is the number of generated trajectories. Unless noted, the sampler is the 256-step stochastic sampler in half precision.}
\label{tab:runs}
\footnotesize\setlength{\tabcolsep}{4pt}
\begin{tabular}{llrl}
\toprule
model & sampler and modification & images & purpose \\
\midrule
EDM uncond. & -- & 14{,}500 & main comparison (Sec.~\ref{sec:edm}) \\
EDM uncond. & -- (replicate) & 3{,}500 & replication \\
EDM uncond. & -- & 320 (+46{,}080 restarts) & committors (Sec.~\ref{sec:cifar}) \\
EDM cond. & equal labels & 4{,}000 & proportions right by design \\
EDM uncond. & Heun, 18 steps, full precision & 10{,}000 & standard sampler \\
EDM uncond. & 256 / 512 / 1{,}024 steps, full prec. & 2{,}000 / 1{,}500 / 1{,}000 & sampler (Sec.~\ref{sec:sampler}) \\
EDM uncond. & exact correction, 2 rounds & 2{,}500 each & correction (Sec.~\ref{sec:correct}) \\
EDM uncond. & first-order corr.; imposed imbalance & 3{,}500 each & pipeline check \\
EDM uncond. & score $\times\lambda$, $\lambda\in\{0.995,0.98,0.95\}$ & 2{,}500 each & corruption (App.~\ref{app:pretrained}) \\
EDM uncond. & noise level $\times\kappa$, $\kappa\in\{1.1,1.25\}$ & 2{,}500 each & corruption \\
DDPM & wrapped as EDM denoiser & 3{,}000 & second family (Sec.~\ref{sec:ddpm}) \\
\bottomrule
\end{tabular}
\end{table}

\para{Interventions.}
The class-proportion correction of Equation~\eqref{eq:tilt} uses the census classifier at every sampling step. Its target is the classifier's class proportions on the test set, and the model's proportions come from the main run. The second round raises the target for each class by the ratio by which the first round fell short. The graded corruptions multiply the learned score by $\lambda$, or pass the denoiser the noise level $\kappa\sigma$ instead of $\sigma$, at every step (Table~\ref{tab:runs}).

\para{Planted errors.}
Within each class we reweight test images by $e^{\beta u}$, where $u$ is the census classifier's standardised margin (true-class logit minus the largest other logit) at the lowest recorded noise level. Weights are normalised so that each class keeps its total weight. We use $\beta\in\{0.25,0.5,1\}$ for all classes and $\beta\in\{0.5,1,2\}$ for deer alone. The distorted half of the test set is compared with the other, undistorted half.

\para{Statistics.}
Per-pair values are medians of bootstrap resamples within each class: 100 for the main comparisons, 60 for planted errors and 30 for corruptions. Intervals for anything averaged or correlated over pairs come from 1{,}000 class-block resamples, which resample the ten classes with replacement and weight each pair by the product of its classes' multiplicities. Census intervals per pair use 1{,}000 resamples. Multiple testing over pairs is controlled with the Benjamini--Hochberg procedure at 5\%. Reliability ceilings use split halves and the Spearman--Brown formula. Model and data are always compared at the same sample size.

\para{Image quality.}
FID \citep{heusel2017gans} uses Inception pool features of the generated images against the 50{,}000 training images. Because FID is biased upwards at small sample sizes, we compare FID only between runs of the same size.

\para{Hardware and software.}
All experiments ran on a single NVIDIA T4 GPU (16 GB) in cloud notebooks, about 80 GPU-hours in total, using PyTorch. The statistical re-analyses ran on CPU. Appendix~\ref{app:compute} breaks down compute and lists dataset and checkpoint licences.

\section{How CIFAR-10 Commits}
\label{sec:cifar}

\begin{figure}[t]
\centering
\includegraphics[width=\textwidth]{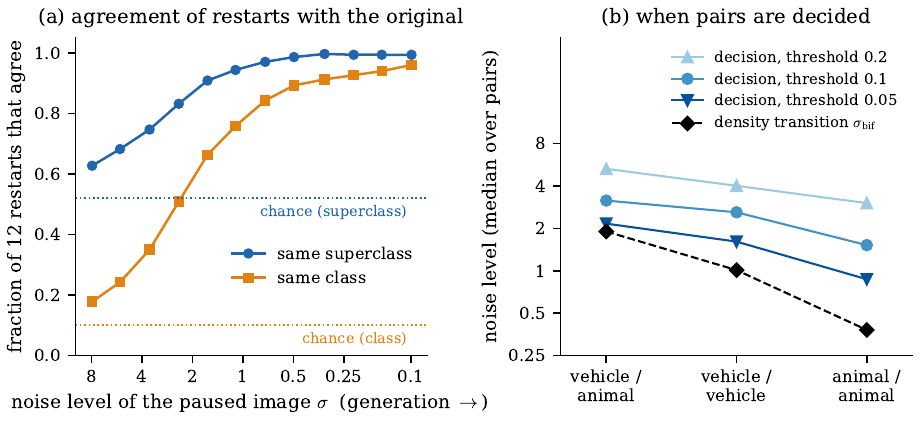}
\caption{Decisions of individual images in the unconditional EDM model. (a) For images paused at each noise level (high noise on the left), the fraction of 12 restarts that end in the same superclass (blue) or the same class (orange) as the original. Superclass agreement rises first. (b) Median noise level at which pairs of each kind are decided, for three thresholds on the cross-over rate (blue), and the median density transition of the same pairs (black). Decisions and density transitions agree on the order, coarse before fine, but not on the absolute noise level.}
\label{fig:committor}
\end{figure}

\begin{table}[t]
\centering
\caption{When pairs are decided, and how well the cheap density transition tracks it. Left: median decision noise level by kind of pair, for three thresholds on the cross-over rate, and the ratio between vehicle/animal and animal/animal pairs. Right: rank correlation between decision levels and density transitions (class-block 95\% interval), and their median ratio, under two choices of measuring stick.}
\label{tab:decisions}
\footnotesize\setlength{\tabcolsep}{3.5pt}
\begin{tabular}{c|cccc|cc|cc}
\toprule
 & \multicolumn{4}{c|}{decision noise level (median)} & \multicolumn{2}{c|}{stick: log-odds $z$} & \multicolumn{2}{c}{stick: $\tanh(z/2)$} \\
threshold & veh./anim. & veh./veh. & anim./anim. & ratio & rank corr. & ratio & rank corr. & ratio \\
\midrule
0.05 & 2.16 & 1.61 & 0.87 & 2.5 & 0.74 [0.47, 0.91] & 1.6 & 0.73 [0.38, 0.93] & 0.39 \\
0.10 & 3.15 & 2.60 & 1.52 & 2.1 & 0.84 [0.66, 0.93] & 2.8 & 0.88 [0.71, 0.97] & 0.58 \\
0.20 & 5.29 & 4.01 & 3.02 & 1.8 & 0.81 [0.54, 0.94] & 4.2 & 0.87 [0.61, 0.96] & 0.93 \\
\bottomrule
\end{tabular}
\end{table}

\para{Coarse decisions come first.}
Individual images settle ``vehicle or animal?'' below a median noise level of $2.4$, and their exact class below $1.1$ (Figure~\ref{fig:committor}a). Pair by pair, vehicle/animal pairs such as airplane/frog are decided at 1.8--2.5 times the noise level of animal/animal pairs such as cat/deer, depending on the threshold (Table~\ref{tab:decisions}). This coarse-to-fine order agrees with earlier observations on trained models \citep{handke2025measuring,sclocchi2025phase} and with theory for data with nested classes \citep{achilli2026theory}. A pair can sometimes be decided before the coarser question. An image may have ruled out ``automobile'' while ``truck'' and ``deer'' are both still open.

\para{The cheap measurement gets the order right but not the timing.}
Density transitions and decision levels rank the 45 pairs in nearly the same order, with rank correlations of 0.73--0.88 for every threshold and both measuring sticks. Their ratio, however, depends on the stick. With the log-odds, density transitions come 1.6--4.2 times later than decisions, and more so for fine distinctions. With the squashed stick $\tanh(z/2)$, which is the pairwise committor scale, the ratio is the same for all kinds of pairs and close to one at the loosest threshold. So the log-odds census in Figure~\ref{fig:census}a, where vehicle/animal splits happen $5.8$ times earlier than animal/animal splits (class-block interval $[3.9,8.1]$), exaggerates how strongly hierarchical commitment is. The \emph{order} is real. Most of it can be predicted from simple image statistics: a Gaussian model built from each class's average image and pixel covariance predicts the transitions with $r=0.89$.

\para{Splits are only mildly lopsided.}
The typical $|h|$ is $0.08$, the lopsidedness of a 54/46 class split. After correcting for testing 44 pairs at once (Benjamini--Hochberg at 5\%), 19 pairs are more lopsided than the noise floor, and 12 remain so, with the same sign, under two different measuring sticks (Figure~\ref{fig:census}b).

\begin{figure}[t]
\centering
\includegraphics[width=\textwidth]{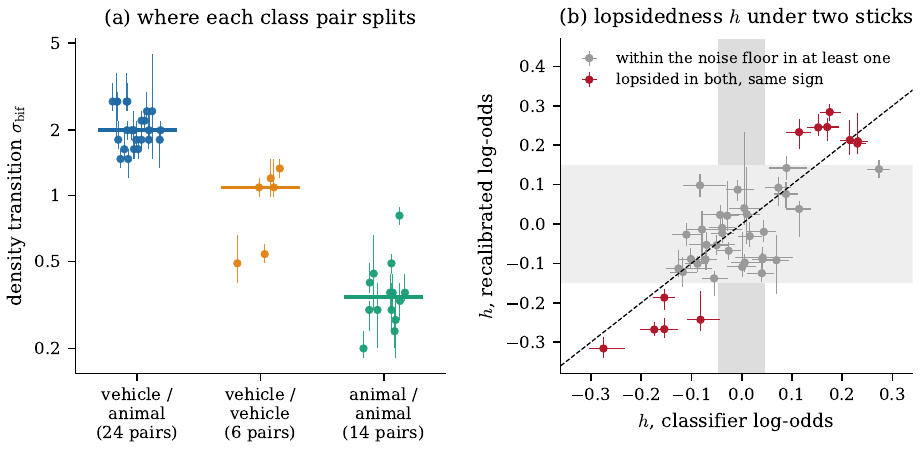}
\caption{The population view of CIFAR-10, measured with the classifier's log-odds. (a) Noise level of the density transition for each class pair, grouped by kind of pair, with 95\% bootstrap intervals; horizontal bars are medians. (b) Lopsidedness $h$ of each pair under the raw log-odds and a recalibrated version. Shaded bands show the noise floor from exactly symmetric pairs. Red points are lopsided beyond the floor under both, with the same sign.}
\label{fig:census}
\end{figure}

\section{What the EDM Model Gets Wrong}
\label{sec:edm}

We now turn to NVIDIA's pretrained unconditional EDM model for CIFAR-10 \citep[56M parameters;][]{karras2022elucidating}. We sample it with the stochastic sampler (256 Euler--Maruyama steps, half precision), and also study its class-conditional sibling. The main run has 14{,}500 generated images, the others 2{,}500--4{,}000. Each is compared with the data at the same sample size.

\begin{figure}[t]
\centering
\includegraphics[width=\textwidth]{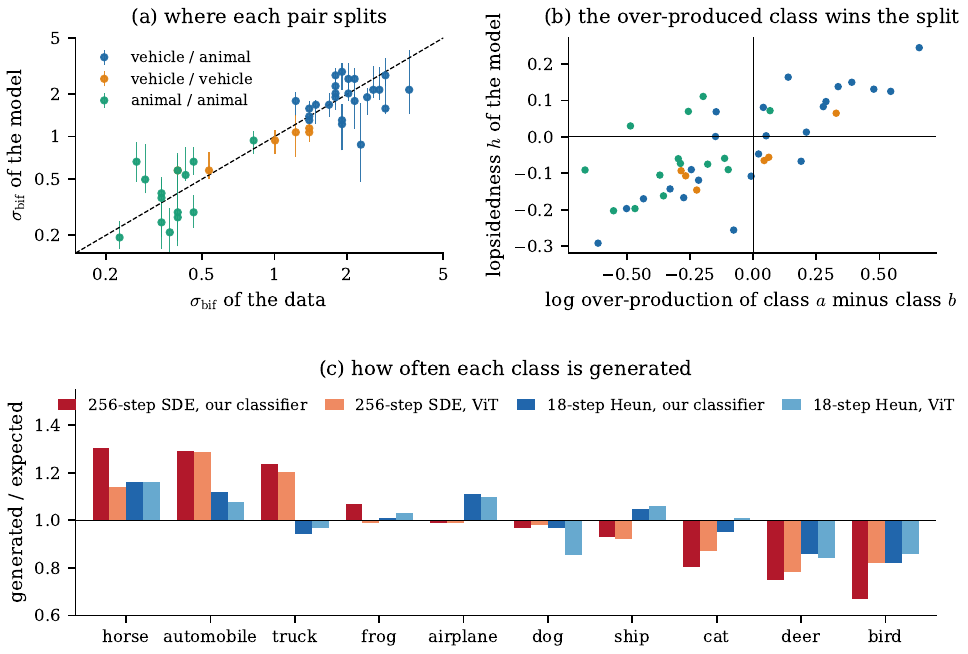}
\caption{The unconditional EDM model, before any correction. (a) Density transition noise levels of the model against the data's, with 95\% bootstrap intervals; the dashed line is equality. (b) The model's lopsidedness $h$ for each pair against the difference in how much the model over-produces the two classes. (c) How often each class is generated, relative to the test set, under our 256-step sampler and EDM's standard 18-step sampler, as judged by our classifier and by an independent ViT.}
\label{fig:pretrained}
\end{figure}

\subsection{At first sight, the model looks only partly right}

The model generates some classes too often (horse, automobile) and others too rarely (bird, deer), as both classifiers agree (Figure~\ref{fig:pretrained}c). Its split noise levels scatter around the data's, with no overall shift but a large spread across pairs: the standard deviation of the log ratio is $0.33$, against $0.21$ between two halves of the data. Its lopsidedness is essentially unrelated to the data's ($r=0.11$). Instead it follows the model's own class imbalance: the class it over-produces is the class that wins the split ($r=0.76$, Figure~\ref{fig:pretrained}b). The timing and lopsidedness errors are also linked as Equation~\eqref{eq:shift} predicts (Figure~\ref{fig:direct}a). Their correlation is $-0.80$ (class-block interval $[-0.92,-0.48]$), and $-0.68$ when the two are measured on separate halves of the generated images, which rules out shared noise. The equation's prediction, which has no free parameters, tracks the measured timing errors with $r=0.71$, though with a slope of $0.65$ that we cannot fully explain. As Section~\ref{sec:background} warned, all of this is what wrong class proportions would produce, whatever else is wrong. So it does not settle the question.

\subsection{The decisive test: give the data the model's class proportions}

\begin{figure}[t]
\centering
\includegraphics[width=\textwidth]{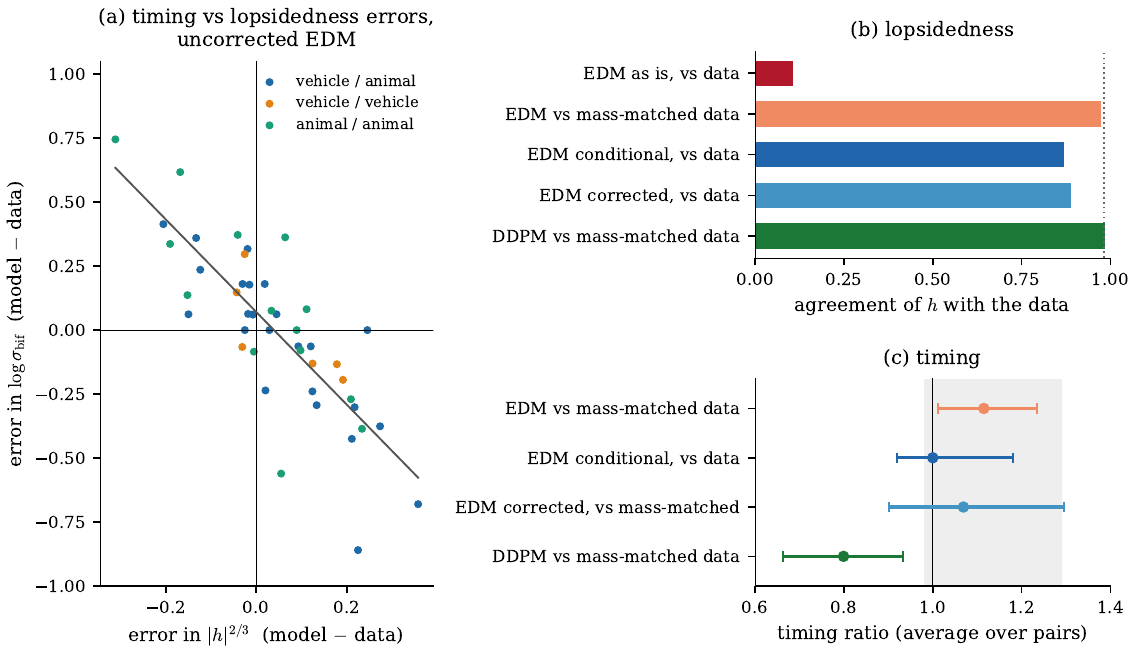}
\caption{Is the class imbalance the only error? (a) For each class pair of the uncorrected EDM model, the error in split timing against the error in lopsidedness, coloured by kind of pair. (b) How well the lopsidedness agrees with the data, for EDM as is, EDM against data reweighted to its class proportions (mass-matched), the class-conditional EDM with equal labels, EDM after correcting its proportions in the sampler, and DDPM against mass-matched data. The dotted line is the agreement between two halves of the data. (c) Average ratio of split noise levels over pairs, with class-block 95\% intervals. The grey band is the range expected between two halves of the data.}
\label{fig:direct}
\end{figure}

\begin{table}[t]
\centering
\caption{Models compared with their data. ``Mass-matched'' data have been reweighted so that their class proportions equal the model's. ``Two halves of the data'' gives the agreement expected from sampling noise alone, at a matching sample size. Brackets are class-block 95\% intervals.}
\label{tab:models}
\footnotesize\setlength{\tabcolsep}{4pt}
\begin{tabular}{llccc}
\toprule
model & compared with & agreement of $h$ & spread of timing error & timing ratio \\
\midrule
EDM, $n{=}14{,}500$ & the data & 0.11 [$-0.47$, 0.68] & 0.332 [0.16, 0.42] & 1.00 \\
 & mass-matched data & \textbf{0.974} [0.94, 0.99] & \textbf{0.178} [0.09, 0.25] & 1.115 [1.01, 1.23] \\
 & two halves of the data & 0.981 [0.93, 0.99] & 0.208 [0.12, 0.30] & --- \\
\midrule
EDM cond., equal labels & the data & 0.87 & 0.223 & 1.00 [0.92, 1.18] \\
\quad $n{=}4{,}000$ & two halves of the data & --- & 0.211 & 1.07 [0.98, 1.29] \\
\midrule
EDM, corrected sampler & the data & 0.889 [0.64, 0.98] & 0.254 [0.15, 0.35] & 1.017 \\
\quad $n{=}2{,}500$ & mass-matched data & 0.946 [0.77, 0.99] & 0.239 [0.14, 0.30] & 1.069 [0.90, 1.30] \\
\midrule
DDPM, $n{=}3{,}000$ & the data & 0.52 [$-0.38$, 0.95] & 0.405 [0.22, 0.53] & 0.60 \\
 & mass-matched data & \textbf{0.985} [0.89, 1.00] & 0.307 [0.13, 0.45] & \textbf{0.799} [0.66, 0.93] \\
 & two halves of the data & 0.980 & 0.206 & --- \\
\bottomrule
\end{tabular}
\end{table}

Here is the test. Take the real data and reweight it so that each class appears as often as it does in the model's samples. Nothing about the model changes. If class proportions are the model's only commitment error, these mass-matched data should behave exactly like the model, pair by pair. They do (Table~\ref{tab:models}). The lopsidedness agrees at $0.974$, the same as the $0.981$ between two halves of the data. The spread of timing errors is $0.178$, at or below the $0.208$ between two halves of the data. The only thing left over is a small shift: the model splits about 11\% earlier in generation than mass-matched data ($[1.01,1.23]$). Our planted-error test tells us how small that is: the same shift would come from distorting every class by only about $0.03$ nats. After the proportions are corrected inside the sampler (Section~\ref{sec:correct}), the shift is $1.02$ against the raw data and $1.07$ ($[0.90,1.30]$) against mass-matched data. That interval contains both 1 and the uncorrected value, so we cannot tell whether the shift survives the correction.

\subsection{A model that gets the proportions right by design}

The class-conditional EDM model can be told which class to draw. If we ask for equal numbers of each class, the class proportions are right by construction: their error by the ViT is $0.016$, below the $0.04$ expected from sampling noise at this sample size. With no reweighting of any kind, this model splits at the same noise levels as the data (median ratio $1.00$, $[0.92,1.18]$), with a pair-to-pair spread of $0.223$ against $0.211$ between two halves of the data. Its lopsidedness agrees with the data's at $0.87$.

\subsection{Is it the network or the sampler?}
\label{sec:sampler}

\begin{table}[t]
\centering
\caption{Class imbalance (average $|\log|$ of generated over expected frequency, by the ViT) of the same EDM network under different samplers. ``Noise floor'' is the value expected from random sampling alone at that sample size, with perfectly balanced classes.}
\label{tab:sampler}
\small
\begin{tabular}{lcccc}
\toprule
sampler & precision & images & class imbalance & noise floor \\
\midrule
stochastic, 256 steps & half & 2{,}500--3{,}500 & 0.12--0.15 & 0.040--0.048 \\
stochastic, 256 steps & full & 2{,}000 & 0.171 & 0.054 \\
stochastic, 512 steps & full & 1{,}500 & 0.132 & 0.063 \\
stochastic, 1{,}024 steps & full & 1{,}000 & 0.137 & 0.076 \\
deterministic, 18 steps (EDM default) & full & 10{,}000 & 0.095 & 0.024 \\
\bottomrule
\end{tabular}
\end{table}

How much of the imbalance belongs to the trained network and how much to our sampler? Some to each (Table~\ref{tab:sampler}). Numerical precision makes no difference. Taking more steps does: going from 256 to 512 steps reduces the imbalance, and going on to 1{,}024 steps does not reduce it further, which leaves it close to EDM's standard deterministic sampler once sampling noise is accounted for. The same classes are over- and under-produced throughout. So the network itself has a class imbalance of roughly 10--15\%, and our coarse 256-step sampler roughly doubles it. Since our comparisons concern the images actually generated, they hold either way. ``The model'' in this paper means the network together with its sampler. Similar skews have been reported for attributes such as gender and race in unconditional diffusion models trained on balanced face data \citep{perera2023analyzing,pal2023gaussian}. Here we see the same effect for classes, and we see exactly where it shows up during generation.

\subsection{Does the model's committor match the data?}
\label{sec:fidelity}

\begin{figure}[t]
\centering
\includegraphics[width=\textwidth]{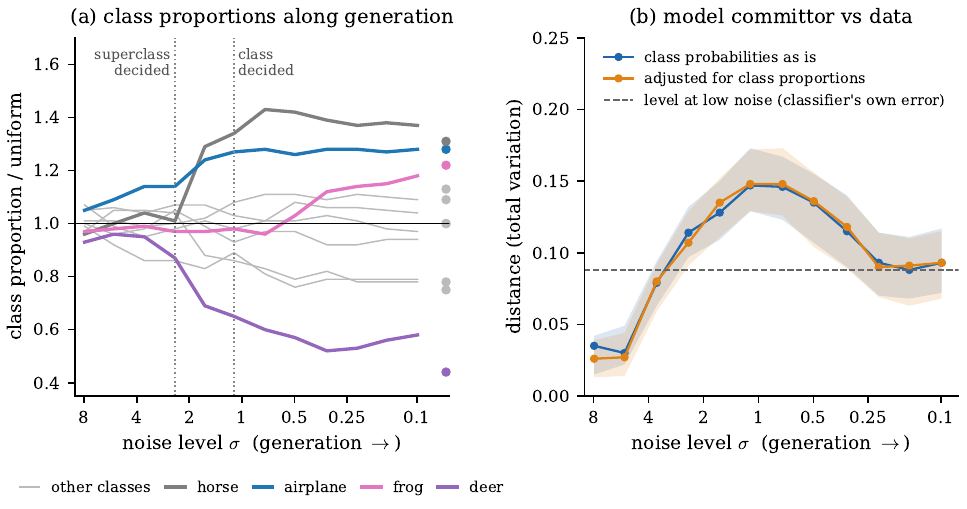}
\caption{The model's decisions, one image at a time. (a) Along the model's own trajectories, the class proportions that our classifier sees at each noise level, relative to uniform. Dots on the right are the final proportions. The imbalance becomes visible as superclasses are decided (dotted lines). (b) How far the model's committors are from the data's class probabilities, beyond what 12 restarts can resolve, with 95\% bootstrap bands: as is (blue), and after adjusting for the model's class proportions (orange), which makes no difference. The dashed line is the level reached at low noise, which reflects our classifier's own mistakes.}
\label{fig:massnoise}
\end{figure}

For a perfect model, the committor of every half-finished image equals the class probability that our noise-aware classifier assigns to it. On average the two agree well: regressing one on the other gives slopes between $0.96$ and $1.11$ at every noise level. Their distance is a little larger than our classifier's own error level, by up to about $0.06$ in total variation, around $\sigma\approx1.1$, which is where decisions are being made (Figure~\ref{fig:massnoise}b). This excess is \emph{not} explained by class proportions: adjusting for them changes it by less than $0.01$. The classifier's own error level also varies with noise level, by about as much, so we cannot say whether this small excess comes from the model or from the classifier. We leave it open. Figure~\ref{fig:massnoise}a shows something else: the model's class imbalance becomes visible to the classifier between $\sigma=2.4$ and $1.6$, as superclasses are being decided. The classifier is weak at high noise, so this tells us where the imbalance becomes \emph{visible}, not necessarily where it starts.

\section{Correcting the Class Proportions}
\label{sec:correct}

If the problem is class proportions, can we fix it during sampling? We can nudge the sampler towards target proportions $f^t$ by adding a small correction to the score at every step. The exact correction depends on the model's own class probabilities. Assuming, as the previous section supports, that the model's classes have the right shapes and only the wrong proportions, those probabilities can be computed from a classifier trained on the data:
\begin{equation}
\nabla_x\Big[\operatorname{LSE}_c\big(\ell_c+\log\tfrac{f^t_c}{f^d_c}\big)-\operatorname{LSE}_c\big(\ell_c+\log\tfrac{f^m_c}{f^d_c}\big)\Big].
\label{eq:tilt}
\end{equation}
Here $\ell_c$ are the classifier's logits, $f^d$ are the class proportions it was trained on, $f^m$ those the model generates, and LSE is log-sum-exp. This is classifier guidance with a strength fixed by theory, so there is nothing to tune. An earlier version of this correction used the data's class probabilities directly, which is right only to first order.

\begin{table}[t]
\centering
\caption{Correcting the class proportions of the unconditional EDM model inside the sampler (2{,}500 images each, 256-step stochastic sampler). Class imbalance by our classifier and by the ViT; the sampling-noise floor at this size is $0.048$. FID at 2{,}500 images.}
\label{tab:correct}
\small
\begin{tabular}{lcccc}
\toprule
run & imbalance (ours) & imbalance (ViT) & FID & $h$ vs data \\
\midrule
no correction & 0.134 & 0.119 & 12.93 & $\approx0$ \\
first-order correction, 1 round & 0.076 & 0.083 & 12.68 & 0.56 \\
exact correction (Eq.~\ref{eq:tilt}), 1 round & 0.093 & 0.107 & 12.95 & 0.54 \\
exact correction, 2nd round & \textbf{0.051} & \textbf{0.074} & 12.98 & \textbf{0.89} \\
\bottomrule
\end{tabular}
\end{table}

One round of either version removes about half of the imbalance, and the exact version is no better than the first-order one (Table~\ref{tab:correct}). The toy problem shows the same thing at imbalances of this size (Appendix~\ref{app:testbed}). The remaining half comes from the classifier's own inaccuracy, and a second round, which aims to overshoot by what the first round missed, takes care of it. Our classifier's measured imbalance then falls to the level expected from sampling noise alone, and the ViT's to $0.074$, with no change in FID. And the model's commitment, measured with no reweighting at all, now matches the data's: the lopsidedness agrees at $0.89$ (up from about zero) and the split noise levels at a ratio of $1.02$ (Table~\ref{tab:models}). This is the most direct evidence in the paper that the class proportions were the problem: fixing them in the sampler fixes commitment.

\para{Must the correction run during the whole of generation?}
In the toy problem, applying the correction only during the noisy first part, or only during the cleaner second part, of generation each achieves part of the effect, and the two parts add up to roughly the whole (Table~\ref{tab:bands}). An earlier, smaller CIFAR-10 experiment seemed to show that neither half worked alone. That experiment used one split point and one random seed, and the effect it measured was smaller than its own noise, so we regard it as inconclusive and draw no conclusion about where on CIFAR-10 the correction must act.

\begin{table}[t]
\centering
\caption{Applying the exact correction to only part of generation, on the toy six-class problem (20{,}000 images, mean $\pm$ SD over 3 seeds): fraction of the full correction achieved when correcting only at noise levels above, or only below, a split point $\sigma_1$.}
\label{tab:bands}
\small
\begin{tabular}{lcccc}
\toprule
split point $\sigma_1$ & 0.5 & 1.0 & 2.0 & 4.0 \\
\midrule
correct only for $\sigma\ge\sigma_1$ (noisy part) & $0.96\pm0.03$ & $0.79\pm0.02$ & $0.32\pm0.02$ & $0.08\pm0.02$ \\
correct only for $\sigma<\sigma_1$ (cleaner part) & $0.08\pm0.02$ & $0.43\pm0.02$ & $0.88\pm0.03$ & $0.99\pm0.03$ \\
sum & 1.05 & 1.22 & 1.20 & 1.07 \\
\bottomrule
\end{tabular}
\end{table}

\section{A Second Model: DDPM}
\label{sec:ddpm}

Is this story special to EDM? We ran the original DDPM model for CIFAR-10 \citep{ho2020denoising} through exactly the same pipeline. DDPM was trained with a different noise parameterisation, so we translate: at each noise level we use DDPM's nearest training noise level and convert its noise prediction into a clean-image guess, $D(x,\sigma)=x-\sigma_t\,\varepsilon_\theta(x\sqrt{\bar\alpha_t},t)$ with $\sigma_t=\sqrt{(1-\bar\alpha_t)/\bar\alpha_t}$. The noise levels differ by a median of $0.2\%$, at most $1.3\%$ above $\sigma=0.1$ and $2.6\%$ at the very lowest levels. That is far below the $10\%$ mismatch that visibly changes commitment in our corruption experiments (Appendix~\ref{app:pretrained}). DDPM is much more imbalanced than EDM, with an imbalance of $0.29$ by the ViT: frogs make up 19\% of its images, dogs and ships about 6\% each.

The lopsidedness of DDPM's splits is, again, entirely explained by its class proportions: against mass-matched data the agreement is $0.985$ ($[0.89,1.00]$), as good as between two halves of the data (Table~\ref{tab:models}). The timing is not. Even against mass-matched data, DDPM splits at $0.80$ times the data's noise level ($[0.66,0.93]$), that is, later in generation, and with more pair-to-pair spread ($0.31$ against $0.21$). Our planted-error test showed that timing reacts to distortions of a few hundredths of a nat per class. This looks like that kind of error, in the direction of classes that are broader and overlap more than they should. So the method does not simply pass every model. It tells apart a model whose only commitment error is class proportions from one that also commits late. \citet{handke2025measuring} measured the timing of class decisions in DDPM models on CIFAR-10 with an entropy-based method; our comparison adds a reference point from the data.

\section{Related Work}

\para{Commitment in diffusion models.}
Speciation and critical windows have been analysed for mixtures and for data with nested classes \citep{biroli2023generative,biroli2024dynamical,raya2023spontaneous,li2024critical,sclocchi2025phase,achilli2026theory,marta2026speciation}. \citet{stancevic2026information} study the lopsided two-class case and entropy-based detectors. \citet{handke2025measuring,handke2026entropic} track class probabilities along the trajectories of trained models. \citet{sclocchi2025phase} use forward--backward experiments, which are close in spirit to our restarts. Our contributions are measurements: committors placed side by side with density transitions, a model-versus-data comparison that is exact in principle and whose power we measure, and a breakdown of a trained model's commitment errors into class proportions and, for a weaker model, timing.

\para{Measuring sticks and committors.}
That a landscape depends on the coordinate in which it is drawn is well known in molecular simulation \citep{hartmann2007comment}, where the committor is the standard coordinate-free alternative \citep{e2010transition}. Following the peaks of a distribution as it is smoothed goes back to \citet{silverman1981using} and \citet{minnotte1993mode}.

\para{Class balance in generative models.}
Counting the classes of generated images with a separate classifier goes back at least to \citet{santurkar2018classification}. Unconditional diffusion models trained on balanced data have been found to skew attributes such as gender and race \citep{perera2023analyzing,pal2023gaussian}. Rebalancing through guidance is established \citep{qin2023class,parihar2024balancing}. \citet{nicoletti2026interplay} analyse how class spread and imbalance shape what a network learns first. The classes our EDM run over-produces tend to be those with more pixel variance (correlation $0.5$ over ten classes), consistent with their analysis though not a test of it. Exposure bias \citep{ning2024elucidating} concerns mismatches between the noisy images a sampler visits and those seen in training.

\section{Limitations}

\begin{itemize}[leftmargin=*]
\item \emph{Scope.} One dataset with ten classes, and two model families. Whether the same picture holds for larger datasets with deeper class hierarchies, such as ImageNet, is open.
\item \emph{What the method can see.} Only commitment along the ``which class?'' direction, measured through a learned classifier. The lopsidedness measure is insensitive to distortions below about $0.5$ nats per class, whereas the timing measure is much more sensitive. Errors unrelated to class identity are invisible.
\item \emph{Committors.} 320 trajectories, 12 restarts each, and noise levels spaced by a factor of $1.5$. Decision levels move by about one grid step with the definition of ``decided''. The small excess in committor fidelity is unexplained.
\item \emph{Left-overs.} EDM's 11\% timing shift against mass-matched data has no mechanism we can identify, and after correction our intervals cannot tell whether it persists. Equation~\eqref{eq:shift} predicts timing errors about a third larger than measured, also unexplained.
\item \emph{DDPM} is run through our translation layer rather than its native sampler.
\item \emph{Correction.} We could not determine on CIFAR-10 at which noise levels the correction must act. The claim that class ordering is mostly simple image statistics rests on a Gaussian prediction, not on generating from a Gaussian model.
\end{itemize}

\section{Conclusion}

So when does a diffusion model decide what to draw? On CIFAR-10, coarse decisions come first. An image settles whether it is a vehicle or an animal at about twice the noise level at which it settles which animal. And what does a strong model get wrong about these decisions? Remarkably little, along the class axis: it decides exactly as its data do, except that it draws some classes too often. Give the data the model's class proportions and the data reproduce the model. Give the model the data's proportions, through conditioning or a small, tuning-free correction to its sampler, and the model reproduces the data. A weaker model shows the same effect plus a delay in when it commits, which the same method picks up. The cheap way of measuring commitment puts decisions in the right order, but only the expensive way, restarting generation and counting outcomes, tells their time.

\bibliographystyle{abbrvnat}
\bibliography{refs}

\clearpage
\appendix

\section{Theory}
\label{app:theory}

Throughout, $x_\sigma=x_0+\sigma\varepsilon$ with $\varepsilon\sim\mathcal N(0,I_d)$, and $p_\sigma$ is the density of $x_\sigma$. For $\sigma>0$ it is smooth and strictly positive, which justifies the integrations by parts below for smooth test functions $g$ with bounded derivatives.

\subsection{The mean-force identity}
\label{app:identity}

For a linear order parameter $m(x)=u\cdot(x-c)/\ell$ with $\|u\|=1$,
\begin{equation}
\frac{\mathrm d}{\mathrm dm}\log p_\sigma(m)=\E\big[\,\ell\,u\cdot\nabla\log p_\sigma(x_\sigma)\;\big|\;m(x_\sigma)=m\,\big].
\label{eq:mf}
\end{equation}
\emph{Proof.} Since $\nabla_x g(m(x))=g'(m(x))\,u/\ell$, we have $g'(m(x))=\ell\,u\cdot\nabla_x g(m(x))$. Integrating by parts, $\E[g'(m)]=-\E[g(m)\,\ell\,u\cdot\nabla\log p_\sigma(x)]$. In one dimension, $\E[g'(m)]=-\E[g(m)(\log p_\sigma)'(m)]$. Equating the two for every $g$ gives~\eqref{eq:mf}. $\square$

This is the potential-of-mean-force identity along a linear collective variable \citep{kirkwood1935statistical}. For a smooth nonlinear $\phi$ with $\nabla\phi\neq0$, the same argument with $F=\nabla\phi/|\nabla\phi|^2$ gives $\frac{\mathrm d}{\mathrm dm}\log p_\sigma(m)=\E[F\cdot\nabla\log p_\sigma+\nabla\cdot F\mid\phi=m]$ \citep{carter1989constrained}. Along a classifier-defined coordinate the divergence term makes the estimator very noisy. At the largest budget we could afford, two independent halves of the data disagreed with each other as much as with the reference, which is why the main text uses Proposition~\ref{prop:joint} instead.

\subsection{Why modes are not fixed points for nonlinear coordinates}
\label{app:modes}

Under the reverse dynamics, a nonlinear $\phi$ evolves with a drift and a diffusion coefficient. The Markovian projection onto $m$ has an effective diffusion $D(m)\propto\E[|\nabla\phi|^2\mid\phi=m]$ \citep{gyongy1986mimicking}. For a one-dimensional diffusion with drift $\mu(m)$ and diffusion $D(m)$, the stationary density satisfies $\mu=D\,(\log p)'+D'$. Zeros of the drift therefore coincide with modes of $p$ only when $D$ is constant, which holds for linear coordinates. For classifier coordinates, $|\nabla\phi|$ varies strongly with $m$, so we interpret the natural-image results as statements about the modes of a density, not about attractors.

\subsection{Proof of Proposition~\ref{prop:joint}, and the committor}
\label{app:joint}

Write the forward process as $x_s=x_0+W_{s^2}$ with $W$ a standard Brownian motion, so that $x_s$ has density $p_s$ and independent increments. Assume the score is exact and that the reverse SDE is started from $p_{\sigma_{\max}}$. By the time-reversal theorem \citep{anderson1982reverse,haussmann1986time}, the reverse-time SDE $\mathrm dx=-2s\nabla\log p_s(x)\,\mathrm ds+\sqrt{2s}\,\mathrm d\bar W_s$, integrated from $\sigma_{\max}$ down to $0$, has the same law on paths as the forward process read backwards. Hence for $\sigma>\sigma_{\min}$ the pair $(x_\sigma,x_{\sigma_{\min}})$ has the same joint law under both. On the forward side this joint law is realised by $x_{\sigma_{\min}}=x_0+\sigma_{\min}\varepsilon_1$ followed by $x_\sigma=x_{\sigma_{\min}}+\sqrt{\sigma^2-\sigma_{\min}^2}\,\varepsilon_2$. The quantity $\E[g(\phi(x_\sigma))\,w(x_{\sigma_{\min}})]$ is a functional of this joint law, so it takes the same value on both sides for every bounded $g$. Normalising by $\E[w(x_{\sigma_{\min}})]$ proves the proposition. $\square$

Two remarks. In practice the sampler starts from $\mathcal N(0,\sigma_{\max}^2I)$ rather than $p_{\sigma_{\max}}$; at $\sigma_{\max}=80$ on the $[-1,1]$ pixel scale the difference is negligible. The probability-flow ODE shares only one-time marginals. It is a deterministic map $x_\sigma\mapsto T(x_\sigma)$, so its joint law sits on the graph of $T$, while the forward joint law has a non-degenerate conditional distribution.

\emph{The committor.} Taking $\sigma_{\min}\to0$ and $w$ the indicator of class $a$, the same argument shows that the probability that a reverse-SDE trajectory through $x$ at noise level $\sigma$ ends in class $a$ equals $P(a\mid x_\sigma=x)$. The class posterior is therefore the committor of the generative process, and the Bayes log-odds is its logit.

\subsection{The cusp normal form: fold condition and jump}
\label{app:cusp}

The zeros of $f(m)=h+am-bm^3$ solve $m^3+pm+q=0$ with $p=-a/b$ and $q=-h/b$. There are three real roots exactly when $-(4p^3+27q^2)>0$, that is, $4a^3>27\,b\,h^2$. The condition is invariant under $f\mapsto\lambda f$. At equality there is a double root at $m_\ast=-\mathrm{sgn}(h)(|h|/2b)^{1/3}$, the newborn mode and its valley. Because the three roots sum to zero, the persisting mode sits at $-2m_\ast$, a distance $3(|h|/2b)^{1/3}$ away. For $h>0$ the persisting root is positive, so the class at $+1$ persists.

\subsection{How class mass enters $h$ and $\sbif$}
\label{app:tanh}

\emph{The asymmetry.} Along an exact log-likelihood ratio $z$, $p(z\mid a)=e^{z}p(z\mid b)$, so the pair density is $p(z\mid b)(w_b+w_ae^{z})$ and its log-slope is $(\log p(z\mid b))'+\mathrm{sigmoid}(z+\log(w_a/w_b))$. A change of weights changes only the sigmoid term, by $\frac14\Delta\log(w_a/w_b)$ at $z=0$. In the two-Gaussian model with classes at $\pm1$ and noise variance $\tau$, $z=2m/\tau$, so in units of $m$ the shift is $\Delta\log(w_a/w_b)/(2\tau)\approx\frac12\Delta\log(w_a/w_b)$ near the transition, where $\tau\approx1$. Equivalently, $h\approx\tanh\eta$ with $\eta=\frac12\log(w_a/w_b)$, so $|h|$ can be read as the asymmetry of a class split $e^{2\eta}:1$. The relation is a property of the coordinate. It holds for any pair of class-conditional densities, which is why we use it as a pipeline check and not as a finding.

\emph{The location.} By the fold condition $4a^3=27bh^2$, the fold happens at $a_f(h)=3(b h^2/4)^{1/3}$. If near the transition $a\approx a_c+s(\log\sigma_c-\log\sigma)$ with $s>0$, then $\log\sbif=\log\sigma_c-a_f(h)/s$. For two distributions with the same $b$, $s$ and $\sigma_c$ but different asymmetries, this gives $\Delta\log\sbif=-3(b/4)^{1/3}\Delta|h|^{2/3}/s$, which is Equation~\eqref{eq:shift}. We estimate $b$ from the cubic fit at the data's transition, and $s$ by regressing the fitted $a$ on $\log\sigma$ over the levels just above and below it.

\subsection{Why partial class tilts are weak}
\label{app:partial}

Let the tilt of Equation~\eqref{eq:tilt} be applied only for $\sigma\ge\sigma_1$. With an exact score and classifier, the marginal at $\sigma_1$ is the tilted one, $p'_{\sigma_1}(x)\propto p_{\sigma_1}(x)\sum_c r_cP_{\sigma_1}(c\mid x)$. Below $\sigma_1$ the untilted dynamics carry each state to class $c$ with probability $P_{\sigma_1}(c\mid x)$. The final frequency of class $c$ is therefore $\E_{p'_{\sigma_1}}[P_{\sigma_1}(c\mid x)]$, which differs from the untilted frequency only through the covariance, under $p_{\sigma_1}$, between $P_{\sigma_1}(c\mid x)$ and the tilt weight $\sum_{c'}r_{c'}P_{\sigma_1}(c'\mid x)$. While posteriors are diffuse that covariance is small, so a high-noise-only tilt corrects little. Conversely, if the tilt acts only for $\sigma<\sigma_1$, samples arrive at $\sigma_1$ from $p_{\sigma_1}$. The tilt then reweights each state's future by $r_c$ in proportion to $P_{\sigma_1}(c\mid x)$, which can only redistribute probability among the classes a state has not yet excluded. Each partial tilt therefore achieves only part of the correction, and the full correction needs both.

\section{Validation on Exactly Solvable Mixtures}
\label{app:controls}

\para{Two Gaussians, linear coordinate.}
For means $\pm s$ with $s\in\{\pm1\}^d$ and weights $w,1-w$, the stationary points of the marginal of $m=s\cdot x/d$ solve $m=\tanh(m/\tau+\eta)$ with $\tau=\sigma^2/d$. With the analytic score in $d=256$, the estimator of Equation~\eqref{eq:mf} recovers these points to within $\pm0.005$ at every noise level, including the fold. For $w=0.8$ the fold is at $\tau=0.450$, and the new mode is born at $m=-0.741$.

\para{Three Gaussians, the reverse-SDE test.}
Proposition~\ref{prop:joint} can only be tested with at least three classes. With two, the pair weights sum to one, the weighted marginal is the plain marginal, and the ODE preserves that too. We used $d=256$ and three classes with sign-pattern means, an isotropic component, and class-specific spreads along each class direction. The order parameter is the exact log-odds of two of the classes, and the terminal weights are the exact posteriors at $\sigma_{\min}=0.05$. With 20{,}000 trajectories and a reference subsample of the same size, ten resamples give a transition ratio of $1.000$ $[0.939,1.065]$ for the reverse SDE at 256 and at 512 steps, and $1.361$ $[1.202,1.536]$ for a 64-step Heun ODE.

\para{Recalibration on known warps.}
We applied the monotone recalibration of Section~\ref{sec:density} to known distortions of an exact log-ratio (Table~\ref{tab:calval}). It recovers affine, $\tanh$ and square-root warps closely, and a strongly convex cubic warp only partly. The ``map nonlinearity'' it reports for a purely affine warp, $0.074$, is the noise floor of that diagnostic at this sample size.

\begin{table}[h]
\centering
\caption{Recalibration of known warps $u=g(z^\ast)$ of an exact log-ratio ($8{,}000$ samples, cross-fitted).}
\label{tab:calval}
\small
\begin{tabular}{lcccc}
\toprule
warp & tilting slope, raw & tilting slope, recalibrated & correlation with $z^\ast$ & relative error \\
\midrule
affine & 1.42 & 0.98 & 0.999 & 0.054 \\
$3\tanh(z/3)$ & 1.34 & 0.99 & 0.997 & 0.079 \\
signed $\sqrt{|z|}$ & 1.65 & 0.99 & 0.998 & 0.060 \\
$z+0.08z^3$ & 0.49 & 0.78 & 0.959 & 0.485 \\
\bottomrule
\end{tabular}
\end{table}

\section{An Exactly Solvable Test Bed}
\label{app:testbed}

\para{Setup.}
Six Gaussian classes in $d=32$ with means $\mu_c$ and isotropic variances $s_c^2\in\{0.30,0.45,0.60,0.35,0.50,0.65\}^2$. Classes 0--2 and 3--5 form two superclasses separated by $\pm3$ along the first axis. Within a superclass the classes sit at the vertices of a triangle of radius $1.6$ in the next two axes, and each class has a small private offset of $0.6$ along its own axis. Under noise, $p_\sigma(x\mid c)=\mathcal N(\mu_c,(s_c^2+\sigma^2)I)$, so the score, the posterior and the committor are all closed-form. The ``data'' have uniform class weights. The ``model'' has the same shapes with weights $(0.22,0.12,0.16,0.20,0.10,0.20)$, a pure mass error of $0.238$ in mean $|\log f/f^d|$. We sample with Euler--Maruyama on a 256-step EDM schedule from $\sigma=40$ to $0.02$.

\para{Tilts.}
A full-trajectory tilt towards uniform reduces the class-mass error from $0.238$ to $0.047$ with the model's own posterior (exact), to $0.045$ with the data posterior and $r=f^t/f^m$ (first order), and to $0.047$ with the data posterior shifted by $\log(f^m/f^d)$, which is identical to the exact form here. At this size of mass error the first-order tilt is as good as the exact one. The residual $0.045$ reflects time discretisation and finite samples. Table~\ref{tab:bands} in the main text gives the band-limited shares.

\para{Committor estimators.}
From 400 model trajectories stopped at 12 levels between 8 and $0.1$, we ran 12 continuations per state. For each of the 15 pairs we computed:
\begin{itemize}[leftmargin=*]
\item the eligibility estimator, a per-state rule that admits a state only if at least 6 of its 12 continuations land in the pair;
\item the pooled cross-over estimator, with threshold $0.1$;
\item the pooled estimator's infinite-sample value, using exact committors;
\item a per-state definition computed with exact committors: the highest level from which at least half of the trajectories have pairwise committor $\ge0.9$.
\end{itemize}
The pooled estimator equals its exact target for all 15 pairs (median log-ratio $0.000$). That target sits about one grid step below the per-state definition (median log-ratio $-0.39$, a factor of $0.68$). The eligibility estimator matches the per-state definition at the median, but reads one grid step too high for 5 of the 15 pairs. In this system no within-superclass pair is decided, by the per-state definition, before the population's superclass decision ($2.43$). On CIFAR-10, by contrast, the threshold-$0.05$ decisions for vehicle pairs ($1.61$) sit below the superclass median ($2.4$), and those at threshold $0.1$ ($2.60$) above it, which is permitted, since a pair can be resolved while other classes remain open.

\section{Coordinate Studies}
\label{app:coord}

\para{Classifiers.}
The main classifier is a ResNet with a 64-channel input convolution and five residual blocks of 64, 128, 128, 256 and 256 channels (two with stride 2), BatchNorm and SiLU. The noise level enters through a sinusoidal embedding of $\frac14\log\sigma$, an MLP, and FiLM scale-and-shift in every block. It is trained on the 50{,}000 CIFAR-10 training images with horizontal flips, $\sigma$ log-uniform on $[0.05,10]$, AdamW (learning rate $2\times10^{-3}$ with a one-cycle schedule, weight decay $5\times10^{-4}$), batch 256, for 25 epochs. Its test accuracy is $85.2\%$, $81.8\%$, $76.0\%$, $64.2\%$, $48.6\%$ and $34.4\%$ at $\sigma=0.05$, $0.25$, $0.5$, $1$, $2$ and $4$. The second architecture is a plain VGG stack without residual connections, with GroupNorm and GELU. It conditions on noise by rescaling the input by $1/\sqrt{1+\sigma^2}$ and appending $\frac14\log\sigma$ as a constant input channel, and its test accuracy is $81.6\%$, $72.9\%$ and $62.5\%$ at $\sigma=0.05$, $0.5$ and $1$. The null classifier has the main architecture with 30 outputs: the ten classes, their vertical flips and their red–blue swaps.

\para{Non-canonical warps.}
Relative to the raw log-odds, re-running the census along $\tanh(z/2)$ moves the transitions by a median factor of $4.5$ and gives $h$ correlated at only $0.67$. A signed square root moves them by a factor of $9.1$, with $h$ correlated at $0.42$. These are the distortions the canonical coordinate is meant to avoid.

\para{Recalibration on real pairs.}
At each pair's transition, the raw log-odds has a tilting slope with median $0.85$ (IQR $0.72$--$0.91$), and the recalibrated coordinate $0.92$ ($0.89$--$0.95$), with median absolute intercept $0.07$. Agreement between the coordinates is summarised below.

\begin{table}[h]
\centering
\caption{Agreement between coordinates over the 44 pairs (correlation, sign agreement of $h$, and median ratio of $\sbif$).}
\small
\begin{tabular}{lccccc}
\toprule
comparison & corr.\ $h$ & sign agr. & median $|\Delta h|$ & corr.\ $\log\sbif$ & median ratio \\
\midrule
ResNet: raw vs.\ recal. & 0.822 & 0.70 & 0.058 & 0.968 & 1.00 \\
VGG: raw vs.\ recal. & 0.834 & 0.77 & 0.056 & 0.946 & 0.90 \\
raw: ResNet vs.\ VGG & 0.938 & 0.86 & 0.027 & 0.958 & 1.11 \\
recal.: ResNet vs.\ VGG & 0.903 & 0.93 & 0.039 & 0.944 & 1.00 \\
\bottomrule
\end{tabular}
\end{table}

\para{The symmetric null.}
Of the twenty exactly symmetric pairs (ten classes, two transformations), thirteen produce a transition; the others never develop two modes. Their $|h|$ has median $0.013$ and maximum $0.047$ through the raw coordinate, median $0.005$ and maximum $0.024$ with the antisymmetrised coordinate, and median $0.063$ and maximum $0.151$ along the recalibrated coordinate. Through the same 30-class classifier, the real pairs have median $|h|=0.072$, and 27 of 44 exceed the null's maximum.

\begin{figure}[h]
\centering
\includegraphics[width=0.66\textwidth]{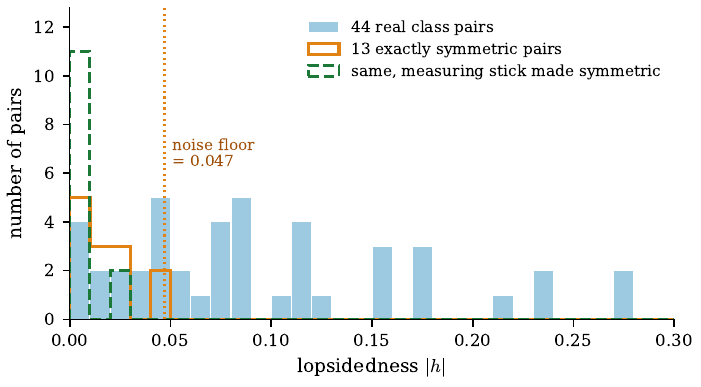}
\caption{The noise floor for lopsidedness. Filled bars: $|h|$ of the 44 real class pairs (census classifier). Orange outline: 13 pairs that are exactly symmetric by construction (a class and its vertical flip or colour-swapped copy), measured through a classifier trained on them. Green dashed outline: the same symmetric pairs with the measuring stick made exactly antisymmetric, which leaves only estimation noise. The dotted line is the largest symmetric-pair value, $0.047$, used as the noise floor.}
\label{fig:null}
\end{figure}

\para{The Gaussian baseline.}
For each pair we took the pixel-space means $\mu_a,\mu_b$ and the pooled covariance $\Sigma$ of the training images, and computed the discriminant separation $d'(\sigma)^2=\sum_k(e_k^\top\Delta\mu)^2/(\lambda_k+\sigma^2)$ over the eigenpairs of $\Sigma$. We then solved for the noise level $\sigma_G$ at which $d'=2$, the separation at which an equal mixture of two equal-variance Gaussians stops being bimodal. $\log\sigma_G$ correlates with $\log\sbif$ at $0.89$ (raw) and $0.82$ (recalibrated). The transitions sit at a roughly constant multiple of $\sigma_G$ (median $6.6$), because the classifier sees far more than second-order structure. After regressing out the trend, the category residuals are $+0.02$ (vehicle/animal), $+0.04$ (vehicle/vehicle) and $-0.23$ (animal/animal) in $\log\sbif$.

\section{Linear Directions on CIFAR-10}
\label{app:linear}

For an unconditional model, the data density along a direction is that of the whole training set. We therefore tested the eleven directions along which the full distribution could plausibly form two lumps: each class against the rest, and vehicles against animals. Their Fisher separabilities are $0.92$--$1.68$. The projections are single humps with shoulders, and the only mode births found were at $\tau=0.01$--$0.07$, where shoulders are smoothed away. For the pretrained model's softest directions we power-iterated the batch-averaged Jacobian of its denoiser at $\sigma\in\{1,2,4,8,16\}$. The top eigenvalue is $0.94$, $0.91$, $0.88$, $0.75$ and $0.48$. The data are unimodal along these directions at every $\sigma$, and the classes they separate best, airplane, frog and ship, share uniform backgrounds.

\section{Statistics}
\label{app:stats}

\para{Class-level tests.}
Each class appears in nine pairs, so pair-level statistics are not independent. For the hierarchy we computed the ratio of median transitions for all $\binom{10}{4}=210$ ways of choosing four ``vehicle'' classes. For the tighter-class tendency we permuted class variances among classes 10{,}000 times. Leaving out one class at a time moves the raw correlation between $-0.39$ and $-0.73$, so no single class drives it. It is not significant at the class level after recalibration ($p=0.08$ for the correlation, $p=0.30$ for the count).

\para{Resampling counts.}
Per-pair point estimates are medians of bootstrap resamples: 100 per pair for the main comparisons (Section~\ref{sec:edm} and Table~\ref{tab:models}), 60 for the planted-error experiments, and 30 for the corruption experiments of Table~\ref{tab:power}. Intervals for quantities pooled over pairs use 1{,}000 class-block resamples. The per-pair census intervals of Table~\ref{tab:census_full} use 1{,}000 resamples, and the pretrained per-pair table uses 100 paired resamples, so its intervals are coarse.

\para{The noise floor of the class-mass error.}
With uniform classes, multinomial noise alone gives a mean $|\log\hat f/0.1|$ of $0.076$, $0.063$, $0.054$, $0.048$, $0.044$, $0.040$ and $0.024$ at $n=1{,}000$, $1{,}500$, $2{,}000$, $2{,}500$, $3{,}000$, $3{,}500$ and $10{,}000$ (2{,}000 simulations each).

\para{Reliability and multiplicity.}
Reliabilities are split-half correlations, stepped up to the full sample with the Spearman--Brown formula $2r/(1+r)$. The ceiling for a correlation between two measured quantities is the geometric mean of their reliabilities. Pair-level tests of $|h|$ against the through-classifier floor use 500 bootstrap resamples, and we control the false discovery rate at 5\% with the Benjamini--Hochberg procedure (22 pairs uncorrected, 19 controlled). Correlations that pool pairs, and ratios averaged over pairs, use a class-block bootstrap that resamples the ten classes with replacement and weights each pair by the product of its classes' multiplicities.

\para{Sensitivity.}
\begin{table}[h]
\centering
\caption{Sensitivity of the census to the density estimator (raw coordinate, 44 pairs), relative to the default: bandwidth $1.5\times$ Silverman, merge distance $0.03$, prominence $0.05$ nats, no persistence requirement.}
\small
\begin{tabular}{lcccc}
\toprule
variant & corr.\ $h$ & median $|\Delta h|$ & corr.\ $\log\sbif$ & median $\sbif$ ratio \\
\midrule
bandwidth $0.75\times$ & 0.964 & 0.030 & 0.978 & 1.22 \\
bandwidth $1.0\times$ & 0.984 & 0.020 & 0.989 & 1.22 \\
bandwidth $2.0\times$ & 0.983 & 0.016 & 0.983 & 0.82 \\
merge $0.06$, prominence $0.10$ & 0.994 & 0.010 & 0.993 & 0.90 \\
persistence over 2 levels & 1.000 & 0.000 & 1.000 & 1.00 \\
persistence over 3 levels & 1.000 & 0.000 & 1.000 & 1.00 \\
\bottomrule
\end{tabular}
\end{table}

\section{Committors: Protocol and Checks}
\label{app:committor}

\para{Protocol.}
We generated 320 trajectories with the same reverse SDE as in the main experiments and stored their states at the 12 schedule steps closest to a logarithmic grid between $\sigma=7.9$ and $0.10$ (the levels of Table~\ref{tab:fidelity}). From each stored state we ran 12 continuations with independent noise to $\sigma_{\min}=0.05$, denoised the end point, and labelled it with the ViT. The trajectories' own final labels give the class frequencies; deer is the rarest at $4.4\%$. At each stored state we also recorded our noise-conditional classifier's posterior.

\para{A biased estimator, and its fix.}
Our first estimator of pairwise decisions considered a state only if at least six of its twelve continuations landed in the pair. At high noise that rule keeps only states that are already heading towards the pair, so it overstates commitment there. It gave median pair decisions of $3.7$--$5.4$, including animal pairs supposedly decided above the noise level at which animal-versus-vehicle is settled ($2.4$), which is impossible. The pooled cross-over rate of Section~\ref{sec:branching} weights each trajectory by how many of its continuations land in the pair, with no threshold. It gives $3.15$, $2.60$ and $1.52$ for the three categories, all below the corresponding coarser decision. Intervals come from 200 bootstrap resamples of trajectories.

\para{Committor fidelity.}
Table~\ref{tab:fidelity} lists, at each level, the total-variation distance between the model's committor and the classifier posterior; the part expected from 12-draw multinomial noise; and the excess above the low-noise plateau, which reflects the classifier's own errors. Excesses below zero at high noise reflect the classifier's diffuse posteriors there, which shows that the floor is not constant.

\begin{table}[h]
\centering
\caption{Committor fidelity of the pretrained unconditional model (320 trajectories, 12 continuations per state). ``Above floor'' subtracts the mean excess at the two lowest noise levels ($0.088$); the interval is from 100 bootstrap resamples of trajectories. ``Slope'' regresses the measured committor on the classifier posterior over all states and classes.}
\label{tab:fidelity}
\small
\begin{tabular}{ccccc}
\toprule
$\sigma$ & total variation & excess over 12-draw noise & above floor [95\% CI] & slope \\
\midrule
7.95 & 0.354 & 0.029 & $-0.059$ [$-0.086$, $-0.037$] & 1.03 \\
5.42 & 0.344 & 0.034 & $-0.054$ [$-0.080$, $-0.027$] & 1.09 \\
3.61 & 0.361 & 0.086 & $-0.002$ [$-0.040$, $+0.015$] & 1.11 \\
2.42 & 0.359 & 0.118 & $+0.029$ [$-0.005$, $+0.049$] & 1.09 \\
1.63 & 0.325 & 0.125 & $+0.037$ [$+0.015$, $+0.062$] & 1.07 \\
1.11 & 0.303 & 0.147 & $+0.058$ [$+0.034$, $+0.089$] & 1.00 \\
0.74 & 0.261 & 0.142 & $+0.054$ [$+0.024$, $+0.083$] & 0.98 \\
0.50 & 0.211 & 0.127 & $+0.038$ [$+0.011$, $+0.061$] & 0.96 \\
0.33 & 0.174 & 0.108 & $+0.019$ [$+0.002$, $+0.039$] & 0.96 \\
0.22 & 0.143 & 0.087 & $-0.001$ [$-0.012$, $+0.015$] & 0.97 \\
0.15 & 0.132 & 0.087 & $-0.001$ [$-0.008$, $+0.002$] & 0.97 \\
0.10 & 0.134 & 0.089 & $+0.001$ [$-0.002$, $+0.008$] & 0.96 \\
\bottomrule
\end{tabular}
\end{table}

\para{Is the excess mass?}
Shifting the classifier posterior by the model's class frequencies, $P_m(c\mid x)\propto P_d(c\mid x)f^m_c/f^d_c$, and recomputing the raw excess changes it by between $-0.007$ and $+0.008$ at the twelve levels. All 95\% intervals of the difference span zero (300 bootstrap resamples of trajectories). The floor cancels in this difference.

\section{The Pretrained Model: Details and Checks}
\label{app:pretrained}

We use the checkpoint \texttt{edm-cifar10-32x32-uncond-vp} \citep{karras2022elucidating} in half precision; its relative deviation from single precision on a test batch is $1.4\times10^{-3}$. The reverse SDE is discretised with Euler--Maruyama, 256 steps on the EDM schedule ($\rho=7$) from $\sigma=80$ to $0.05$, and logits are recorded at 68 levels between $0.086$ and $7.95$. The reference is the 10{,}000 test images with four noise draws each, noised first to $\sigma_{\min}$ (for the terminal weights) and then onward. The standardisation of the log-odds is computed from the reference at each level and applied to both sides. The standard sampler is EDM's deterministic Heun method with 18 steps (35 evaluations) from $\sigma=80$ to $0.002$, in full precision. FID uses Inception features \citep{heusel2017gans} against the 50{,}000 training images. The independent classifier is a ViT-Base fine-tuned on CIFAR-10 at $224\times224$ resolution \citep{dosovitskiy2021image}, with 98.4\% accuracy on 2{,}000 test images.

Class frequencies (our classifier / ViT), relative to the test set, which both classifiers read as uniform to within $0.003$:

\begin{table}[h]
\centering
\small
\caption{Generated class frequencies under two samplers (sample size; labelling classifier).}
\footnotesize\setlength{\tabcolsep}{3pt}
\begin{tabular}{lcccccccccc}
\toprule
 & airpl. & autom. & bird & cat & deer & dog & frog & horse & ship & truck \\
\midrule
SDE, fp16 (14.5k; ours) & 0.100 & 0.128 & 0.067 & 0.081 & 0.075 & 0.094 & 0.108 & 0.129 & 0.094 & 0.125 \\
SDE, fp16 (3.5k; ours) & 0.101 & 0.133 & 0.076 & 0.082 & 0.082 & 0.104 & 0.104 & 0.111 & 0.094 & 0.114 \\
SDE, fp16 (3.5k; ViT) & 0.099 & 0.130 & 0.082 & 0.089 & 0.079 & 0.095 & 0.100 & 0.113 & 0.093 & 0.119 \\
Heun, fp32 (10k; ours) & 0.112 & 0.111 & 0.082 & 0.096 & 0.086 & 0.094 & 0.102 & 0.115 & 0.106 & 0.095 \\
Heun, fp32 (10k; ViT) & 0.110 & 0.109 & 0.086 & 0.103 & 0.085 & 0.083 & 0.104 & 0.115 & 0.107 & 0.096 \\
\bottomrule
\end{tabular}
\end{table}

For the interventions, the imposed tilt (bird $\times2$, automobile $\times\frac12$) moved bird from $0.075$ to $0.096$ and automobile from $0.133$ to $0.104$. That is less than the tilt would achieve with a perfect classifier. The law of Section~\ref{sec:edm} is tested against the frequencies actually measured on the generated images. The degraded run, with the score scaled by $0.85$ at every noise level, produced FID $127.7$ and grossly distorted frequencies (deer $2\%$, truck $20\%$), and no pair developed a transition.

\para{Samplers.}
Class frequencies by the ViT for the same checkpoint (test set read as 0.100, 0.101, 0.100, 0.102, 0.101, 0.096, 0.100, 0.099, 0.100, 0.099). Euler--Maruyama in full precision: 256 steps (0.101, 0.131, 0.070, 0.091, 0.071, 0.089, 0.096, 0.133, 0.099, 0.119); 512 steps (0.102, 0.126, 0.087, 0.096, 0.085, 0.087, 0.099, 0.127, 0.080, 0.111); 1{,}024 steps (0.093, 0.136, 0.083, 0.101, 0.072, 0.088, 0.096, 0.134, 0.098, 0.099). Against data mass-matched to it, the 256-step full-precision run gives asymmetry agreement $0.964$ ($[0.84,1.00]$) and a location ratio of $1.09$ ($[0.88,1.43]$), the same picture as the half-precision runs. FID at these small sample sizes ($17.2$, $22.5$, $32.7$ at $n=2{,}000$, $1{,}500$, $1{,}000$) is dominated by its small-sample bias and cannot be compared across rows.

\para{The corrected tilt.}
Round 1 uses Equation~\eqref{eq:tilt} with $f^t=f^d$ and $f^m$ from the 14{,}500-trajectory run. Round 2 keeps $f^m$ but sets the target to $f^t_{\rm eff}\propto f^{d}\cdot f^{d}/f^{(1)}$, where $f^{(1)}$ are the frequencies achieved by round 1. ViT frequencies after round 2: (0.090, 0.102, 0.095, 0.106, 0.088, 0.096, 0.112, 0.118, 0.092, 0.102).

\para{DDPM.}
\texttt{google/ddpm-cifar10-32} with the linear $\beta$ schedule ($10^{-4}$ to $0.02$, 1{,}000 steps), so $\sigma_t$ ranges from $0.010$ to $157$. At each sampler noise level $\sigma$ we use the training timestep $t$ minimising $|\log\sigma_t-\log\sigma|$, evaluate $\varepsilon_\theta(x\sqrt{\bar\alpha_t},t)$ in half precision, and return $D=x-\sigma_t\varepsilon_\theta$. Over our 256-step schedule, the relative mismatch $|\sigma_t/\sigma-1|$ has median $0.18\%$, maximum $1.3\%$ for $\sigma\ge0.1$, and $2.6\%$ at $\sigma\approx0.06$. ViT frequencies: (0.085, 0.098, 0.071, 0.142, 0.081, 0.062, 0.186, 0.091, 0.064, 0.121); FID $21.4$ at $n=3{,}000$ with our sampler.

\para{Graded corruptions and band-limited corrections.}
Table~\ref{tab:power} lists the corrupted and band-corrected runs (2{,}500 trajectories each). Transitions are compared pair by pair with an uncorrupted run of the same size, using 30 bootstrap medians per pair, and summarised as the geometric mean ratio over pairs with a class-block interval. The first row compares the baseline with an independent resampling of itself.

\begin{table}[h]
\centering
\caption{Corruptions of the score and band-limited class corrections. ``Class-mass error'' is the mean $|\log f/f_{\rm ref}|$ with our classifier's labels.}
\label{tab:power}
\footnotesize\setlength{\tabcolsep}{4pt}
\begin{tabular}{lccccc}
\toprule
run & $\sbif$ ratio, raw & $\sbif$ ratio, rebalanced & rebalanced $h$ vs data & class-mass error & FID \\
\midrule
baseline (self) & 0.98 [0.94, 1.03] & 1.00 [0.97, 1.03] & 0.91 & 0.134 & 12.9 \\
score $\times0.995$ & 0.81 [0.72, 0.90] & 0.94 [0.83, 1.04] & 0.89 & 0.203 & 13.7 \\
score $\times0.98$ & 0.71 [0.57, 0.90] & 0.84 [0.68, 1.11] & 0.94 & 0.227 & 15.7 \\
score $\times0.95$ & 0.55 [0.39, 0.69] & 0.87 [0.75, 0.97] & 0.84 & 0.393 & 25.4 \\
noise level $\times1.1$ & 0.55 [0.44, 0.68] & 0.81 [0.66, 1.05] & 0.85 & 0.349 & 35.7 \\
noise level $\times1.25$ & 0.60 [0.30, 1.50] & 0.66 [0.42, 1.27] & 0.21 & 1.009 & 106.1 \\
\midrule
correction, all $\sigma$ & 0.92 [0.79, 1.02] & 0.90 [0.81, 1.00] & 0.93 & 0.076 & 12.7 \\
correction, $\sigma\ge2.4$ & 0.88 [0.76, 1.03] & 0.89 [0.73, 1.06] & 0.95 & 0.130 & 13.0 \\
correction, $\sigma\le2.4$ & 0.85 [0.72, 0.96] & 0.90 [0.70, 1.10] & 0.93 & 0.138 & 13.3 \\
\bottomrule
\end{tabular}
\end{table}

\begin{figure}[h]
\centering
\includegraphics[width=\textwidth]{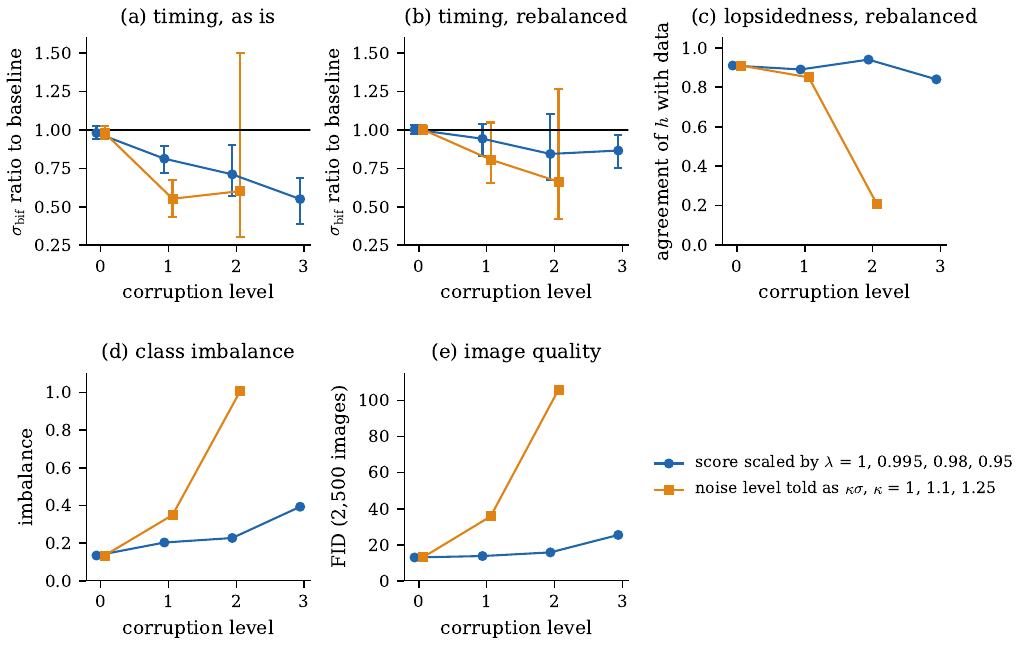}
\caption{Graded corruptions of the score (Table~\ref{tab:power}). Blue: score multiplied by $\lambda$; orange: denoiser told the noise level is $\kappa\sigma$. (a) Ratio of density transitions to the uncorrupted baseline, as sampled, and (b) after rebalancing, with class-block intervals. (c) Agreement of the rebalanced lopsidedness with the data, which illustrates the low sensitivity of this statistic quantified in Section~\ref{sec:power}. (d) Class imbalance. (e) FID at 2{,}500 images.}
\label{fig:power}
\end{figure}

\section{Digits: Linear Coordinates, Small Networks, and Class Bias}
\label{app:mnist}

\para{Census along linear coordinates.}
On MNIST and FashionMNIST, the Fisher direction $u\propto(S_a+S_b+I)^{-1}(\mu_a-\mu_b)$ is a good order parameter, and the data side needs only a histogram of $u\cdot x_0$ convolved with a Gaussian. With 1{,}000 bootstrap resamples and 180 symmetrised controls per dataset, 37 of 45 MNIST pairs (median $|h|=0.026$) and 30 of 45 FashionMNIST pairs (median $0.015$) exceed the floor with intervals excluding zero (Table~\ref{tab:mnist}). Across pairs, the jump at the transition scales as $(|h|/2b)^{0.36}$, close to the normal form's $1/3$. Synthetic pairs with known unfoldings, run through the same pipeline, give the same exponent ($0.366$) and the same 10\% shortfall in prefactor. This exponent therefore reflects the truncated-cubic fit rather than the data, and we do not use it as evidence.

\begin{table}[h]
\caption{MNIST census along the LDA direction. $\tbif$ is in units of $\tau=\sigma^2/\ell^2$.}
\label{tab:mnist}
\begin{center}
\scriptsize\setlength{\tabcolsep}{3pt}
\begin{tabular}{lccccc|lccccc}
\toprule
pair & $\tbif$ & jump & $h$ & $a$ & $b$ & pair & $\tbif$ & jump & $h$ & $a$ & $b$ \\
\midrule
(0,1) & 0.906 & 0.72 & $-0.011$ & $+0.042$ & 0.277  & (2,9) & 0.754 & 1.07 & $-0.041$ & $+0.143$ & 0.350  \\
(0,2) & 0.827 & 0.74 & $+0.011$ & $+0.045$ & 0.272  & (3,4) & 0.792 & 0.95 & $-0.027$ & $+0.101$ & 0.323  \\
(0,3) & 0.864 & 0.63 & $-0.007$ & $+0.026$ & 0.260  & (3,5) & 0.571 & 1.13 & $-0.050$ & $+0.167$ & 0.366  \\
(0,4) & 0.855 & 0.84 & $-0.018$ & $+0.069$ & 0.297  & (3,6) & 0.782 & 0.96 & $-0.028$ & $+0.104$ & 0.325  \\
(0,5) & 0.748 & 1.06 & $-0.040$ & $+0.142$ & 0.354  & (3,7) & 0.770 & 1.07 & $-0.040$ & $+0.140$ & 0.343  \\
(0,6) & 0.810 & 0.86 & $-0.019$ & $+0.074$ & 0.298  & (3,8) & 0.770 & 0.55 & $-0.004$ & $+0.014$ & 0.236  \\
(0,7) & 0.812 & 0.93 & $-0.026$ & $+0.097$ & 0.322  & (3,9) & 0.765 & 1.02 & $-0.034$ & $+0.121$ & 0.328  \\
(0,8) & 0.931 & 0.40 & $-0.002$ & $-0.001$ & 0.231  & (4,5) & 0.835 & 0.59 & $-0.006$ & $+0.020$ & 0.255  \\
(0,9) & 0.880 & 0.70 & $-0.010$ & $+0.038$ & 0.268  & (4,6) & 0.854 & 0.79 & $+0.014$ & $+0.056$ & 0.281  \\
(1,2) & 0.696 & 1.23 & $+0.075$ & $+0.243$ & 0.439  & (4,7) & 0.750 & 0.97 & $+0.029$ & $+0.105$ & 0.320  \\
(1,3) & 0.828 & 1.01 & $+0.033$ & $+0.120$ & 0.329  & (4,8) & 0.797 & 0.93 & $+0.025$ & $+0.094$ & 0.314  \\
(1,4) & 0.856 & 0.92 & $+0.025$ & $+0.092$ & 0.314  & (4,9) & 0.691 & 0.97 & $-0.027$ & $+0.102$ & 0.306  \\
(1,5) & 0.901 & 0.70 & $+0.010$ & $+0.037$ & 0.266  & (5,6) & 0.844 & 0.28 & $+0.000$ & $-0.010$ & 0.222  \\
(1,6) & 0.826 & 0.99 & $+0.032$ & $+0.117$ & 0.337  & (5,7) & 0.857 & 0.47 & $+0.002$ & $+0.004$ & 0.242  \\
(1,7) & 0.829 & 1.01 & $+0.033$ & $+0.120$ & 0.332  & (5,8) & 0.615 & 1.26 & $+0.074$ & $+0.236$ & 0.408  \\
(1,8) & 0.630 & 1.36 & $+0.108$ & $+0.322$ & 0.478  & (5,9) & 0.788 & 0.89 & $+0.020$ & $+0.078$ & 0.294  \\
(1,9) & 0.880 & 0.83 & $+0.017$ & $+0.066$ & 0.292  & (6,7) & 0.906 & 0.52 & $-0.004$ & $+0.011$ & 0.250  \\
(2,3) & 0.766 & 0.81 & $+0.014$ & $+0.057$ & 0.271  & (6,8) & 0.862 & 0.59 & $+0.006$ & $+0.020$ & 0.250  \\
(2,4) & 0.746 & 1.05 & $-0.040$ & $+0.138$ & 0.349  & (6,9) & 0.912 & 0.51 & $-0.004$ & $+0.010$ & 0.248  \\
(2,5) & 0.712 & 1.03 & $-0.037$ & $+0.130$ & 0.341  & (7,8) & 0.804 & 0.95 & $+0.027$ & $+0.099$ & 0.315  \\
(2,6) & 0.756 & 1.01 & $-0.033$ & $+0.119$ & 0.326  & (7,9) & 0.703 & 0.95 & $-0.024$ & $+0.090$ & 0.289  \\
(2,7) & 0.732 & 1.08 & $-0.042$ & $+0.148$ & 0.353  & (8,9) & 0.752 & 1.09 & $-0.043$ & $+0.150$ & 0.349  \\
(2,8) & 0.672 & 1.09 & $-0.043$ & $+0.149$ & 0.350  & & & & & & \\

\bottomrule
\end{tabular}
\end{center}
\end{table}

\para{Small networks commit late.}
We trained EDM U-Nets with about one million parameters on pairs of digits, three seeds for most pairs, and measured their transitions with the mean-force identity (Table~\ref{tab:small}). Early in training they also carry an asymmetry the data do not have. For (0,1), $h=+0.36\pm0.05$ at epoch 20 against $-0.01$ for the data, falling to $0.00\pm0.02$ by epoch 80.

\begin{table}[h]
\centering
\caption{Transition of small trained denoisers along the LDA direction, in noise variance $\tau$ (mean $\pm$ sd over three seeds where given).}
\label{tab:small}
\small
\begin{tabular}{lccc}
\toprule
pair & data & epoch 20 & epoch 80 \\
\midrule
MNIST (0,1) & 0.905 & $0.28\pm0.02$ & $0.74\pm0.03$ \\
MNIST (0,8) & 0.933 & $0.40\pm0.01$ & $0.64\pm0.02$ \\
MNIST (3,8) & 0.778 & not bimodal & 0.58 \\
FashionMNIST (1,7) & 0.917 & 0.28 & 0.75 \\
\bottomrule
\end{tabular}
\end{table}

\para{Where their class bias lives.}
Undertrained networks generate class-imbalanced samples from balanced data. We replaced the network's drift along the class axis with the data's, $D'(x)=D(x)+\sigma^2[g_{\rm data}(m;\sigma)-g_{\rm net}(m;\sigma)]\,u/\ell$, either at all noise levels or in one band at a time (Table~\ref{tab:bands}). A sham replacement, using the network's own drift re-estimated from independent samples, and a replacement along a random orthogonal direction both had no effect. Replacing at all noise levels removed most of the stochastic sampler's bias for two of four pairs, and 19--35\% of the deterministic sampler's bias where that bias was large. No single band dominates, and the shares roughly add up: the bias accumulates over the whole trajectory.

\begin{table}[h]
\centering
\caption{Band-limited drift replacement (epoch 20, three seeds). Excess fraction of the first class over the real-data baseline, mean $\pm$ sd.}
\label{tab:bands}
\footnotesize\setlength{\tabcolsep}{4pt}
\begin{tabular}{llccccc}
\toprule
pair & sampler & none & $\tau>1.5$ & $0.3\le\tau\le1.5$ & $\tau<0.3$ & all $\tau$ \\
\midrule
(0,1) & Heun-40 & $+0.213\pm0.034$ & $+0.186\pm0.035$ & $+0.188\pm0.028$ & $+0.193\pm0.038$ & $+0.138\pm0.024$ \\
(0,1) & SDE-40 & $+0.097\pm0.062$ & $+0.079\pm0.064$ & $+0.056\pm0.056$ & $+0.071\pm0.066$ & $-0.008\pm0.042$ \\
(3,8) & SDE-40 & $+0.101\pm0.022$ & $+0.098\pm0.029$ & $+0.077\pm0.025$ & $+0.073\pm0.031$ & $+0.030\pm0.017$ \\
(0,8) & Heun-40 & $-0.076\pm0.001$ & $-0.065\pm0.006$ & $-0.072\pm0.006$ & $-0.079\pm0.007$ & $-0.061\pm0.010$ \\
(2,7) & Heun-40 & $+0.184\pm0.011$ & $+0.164\pm0.006$ & $+0.172\pm0.002$ & $+0.169\pm0.010$ & $+0.149\pm0.004$ \\
(2,7) & SDE-40 & $+0.102\pm0.007$ & $+0.107\pm0.007$ & $+0.098\pm0.007$ & $+0.101\pm0.010$ & $+0.089\pm0.018$ \\
\bottomrule
\end{tabular}
\end{table}

\section{Does the Fold's Shape Affect Sampling?}
\label{app:sampling}

In one-dimensional reduced dynamics with the exact drift of all 45 MNIST pairs, stochastic samplers with few steps under-populate the persisting class in proportion to $|h|$. Deterministic samplers cannot show any such effect, because a one-dimensional flow preserves quantiles. The reduced dynamics reproduces the marginals of $m$ only in continuous time, however, so we repeated the test in $d=784$ on a two-class mixture with unequal spreads along the class direction and an analytic score, with real samplers and 4{,}000 samples (Table~\ref{tab:fulldim}). No sampler shows a dependence on $h$; all errors stay within about 2\%. We also found that, along real ODE trajectories of small networks, the sign of the denoised class coordinate never changes after the first step for at least three quarters of trajectories.

\begin{table}[h]
\centering
\caption{Class-fraction error towards the persisting class in $d=784$ with an exact score (s.e.\ $\approx0.008$).}
\label{tab:fulldim}
\small
\begin{tabular}{ccccccc}
\toprule
$h$ & Heun-10 & Heun-40 & SDE-10 & SDE-40 & SDE-160 \\
\midrule
$-0.0002$ & $-0.014$ & $-0.014$ & $-0.008$ & $-0.017$ & $-0.015$ \\
$+0.0080$ & $+0.007$ & $+0.010$ & $-0.002$ & $+0.015$ & $+0.011$ \\
$+0.0313$ & $-0.010$ & $-0.004$ & $-0.014$ & $+0.006$ & $+0.002$ \\
$+0.0831$ & $-0.022$ & $-0.018$ & $-0.013$ & $+0.002$ & $-0.007$ \\
$+0.2027$ & $-0.020$ & $-0.015$ & $-0.004$ & $+0.003$ & $+0.000$ \\
\bottomrule
\end{tabular}
\end{table}

\section{Entropy-Variance Detectors}
\label{app:detector}

A common way to locate the critical window is the noise level at which the variance across samples of a conditional entropy peaks \citep{stancevic2026information}. Along MNIST LDA coordinates, the peak of $\mathrm{Var}[H(c\mid m)]$ sits at a median of $0.88\,\tbif$, a VP-time offset of $-0.07$. One pair, (3,6), gives a much larger ratio. Its variance curve rises again above $\tau\approx3$, where the posterior carries no information, and the maximum jumps to the edge of the grid; restricting the search to $\tau<3$ removes the outlier. Detectors of this kind therefore sit close to the population transition, and, by Section~\ref{sec:branching}, somewhat after individual samples have decided.

\section{Compute and Assets}
\label{app:compute}

All experiments ran on a single NVIDIA T4 GPU (16 GB) in a cloud notebook environment, about 80 GPU-hours in total. The main costs were the pretrained-model runs (about 5 hours for 14{,}500 trajectories plus the reference, about 8.5 hours for the first interventions, about 4 hours for branching, about 1.5 hours for the conditional model, about 6.5 hours for the corruption and band experiments, and about 6.6 hours for the sampler, corrected-tilt and DDPM runs) and the coordinate studies (about 5 hours). The statistical re-analyses ran on CPU. CIFAR-10 \citep{krizhevsky2009learning} is distributed by its authors without an explicit licence. MNIST \citep{lecun1998gradient} is available under CC BY-SA 3.0 and Fashion-MNIST \citep{xiao2017fashion} under the MIT licence. The EDM code and checkpoints \citep{karras2022elucidating} are licensed CC BY-NC-SA 4.0, and we use them for non-commercial research only.

\clearpage
\section{Per-Pair Results}
\label{app:tables}

\begin{table}[h]
\centering
\caption{CIFAR-10 census (test set). Transition noise with 95\% interval (raw coordinate); asymmetry $h$ with 95\% interval along the raw log-odds and along the recalibrated log-ratio. ``Robust'' marks pairs asymmetric beyond the through-classifier null in both coordinates, with the same sign. Cat/dog never develops two modes.}
\label{tab:census_full}
\scriptsize
\begin{tabular}{lcccc}
\toprule
pair & $\sbif$ (raw) & $h$ (raw) & $h$ (recalibrated) & robust \\
\midrule
airplane/automobile & 1.33 [1.20, 1.47] & $+0.006$ [$-0.022$, $+0.047$] & $-0.098$ [$-0.123$, $-0.076$] &  \\
airplane/bird & 2.69 [1.80, 3.63] & $+0.010$ [$-0.025$, $+0.036$] & $+0.025$ [$-0.050$, $+0.145$] &  \\
airplane/cat & 1.63 [1.47, 1.99] & $-0.126$ [$-0.151$, $-0.060$] & $-0.113$ [$-0.135$, $-0.067$] &  \\
airplane/deer & 2.43 [1.47, 4.44] & $-0.070$ [$-0.105$, $-0.014$] & $-0.053$ [$-0.116$, $-0.021$] &  \\
airplane/dog & 1.80 [1.63, 2.20] & $-0.079$ [$-0.111$, $+0.000$] & $-0.014$ [$-0.102$, $+0.034$] &  \\
airplane/frog & 2.69 [2.20, 2.97] & $-0.275$ [$-0.303$, $-0.233$] & $-0.316$ [$-0.340$, $-0.286$] & yes \\
airplane/horse & 1.63 [1.47, 1.80] & $-0.110$ [$-0.138$, $-0.082$] & $-0.027$ [$-0.057$, $+0.003$] &  \\
airplane/ship & 0.54 [0.49, 0.60] & $-0.088$ [$-0.112$, $-0.067$] & $-0.101$ [$-0.121$, $-0.081$] &  \\
airplane/truck & 1.09 [0.98, 1.20] & $-0.050$ [$-0.073$, $-0.027$] & $-0.054$ [$-0.080$, $-0.029$] &  \\
automobile/bird & 1.80 [1.63, 2.20] & $-0.043$ [$-0.078$, $-0.003$] & $+0.023$ [$-0.007$, $+0.048$] &  \\
automobile/cat & 1.47 [1.33, 1.63] & $-0.174$ [$-0.203$, $-0.142$] & $-0.268$ [$-0.285$, $-0.250$] & yes \\
automobile/deer & 1.80 [1.63, 1.99] & $-0.154$ [$-0.179$, $-0.126$] & $-0.267$ [$-0.289$, $-0.238$] & yes \\
automobile/dog & 1.80 [1.63, 2.20] & $-0.117$ [$-0.143$, $-0.088$] & $-0.122$ [$-0.160$, $-0.085$] &  \\
automobile/frog & 1.99 [1.80, 1.99] & $-0.154$ [$-0.177$, $-0.133$] & $-0.187$ [$-0.219$, $-0.164$] & yes \\
automobile/horse & 1.63 [1.47, 1.80] & $-0.101$ [$-0.125$, $-0.075$] & $-0.089$ [$-0.114$, $-0.060$] &  \\
automobile/ship & 1.09 [0.98, 1.47] & $-0.083$ [$-0.116$, $-0.032$] & $+0.098$ [$+0.068$, $+0.127$] &  \\
automobile/truck & 0.49 [0.40, 0.66] & $-0.082$ [$-0.109$, $-0.043$] & $-0.243$ [$-0.272$, $-0.171$] & yes \\
bird/cat & 0.30 [0.20, 0.40] & $-0.039$ [$-0.066$, $-0.014$] & $-0.010$ [$-0.038$, $+0.014$] &  \\
bird/deer & 0.20 [0.18, 0.24] & $-0.072$ [$-0.097$, $-0.049$] & $-0.090$ [$-0.118$, $-0.065$] &  \\
bird/dog & 0.30 [0.24, 0.40] & $+0.041$ [$+0.021$, $+0.063$] & $-0.089$ [$-0.135$, $-0.059$] &  \\
bird/frog & 0.36 [0.33, 0.44] & $-0.039$ [$-0.062$, $-0.016$] & $-0.023$ [$-0.048$, $+0.005$] &  \\
bird/horse & 0.30 [0.24, 0.33] & $+0.230$ [$+0.212$, $+0.248$] & $+0.204$ [$+0.177$, $+0.282$] & yes \\
bird/ship & 1.99 [1.80, 2.43] & $-0.029$ [$-0.052$, $-0.006$] & $+0.021$ [$-0.017$, $+0.109$] &  \\
bird/truck & 2.20 [1.80, 2.43] & $+0.016$ [$-0.008$, $+0.040$] & $-0.031$ [$-0.059$, $+0.004$] &  \\
cat/deer & 0.36 [0.30, 0.44] & $-0.008$ [$-0.032$, $+0.024$] & $+0.087$ [$+0.059$, $+0.111$] &  \\
cat/frog & 0.44 [0.30, 0.66] & $+0.042$ [$-0.013$, $+0.102$] & $-0.086$ [$-0.116$, $-0.061$] &  \\
cat/horse & 0.27 [0.18, 0.36] & $+0.273$ [$+0.250$, $+0.295$] & $+0.139$ [$+0.117$, $+0.163$] &  \\
cat/ship & 1.99 [1.80, 2.43] & $+0.114$ [$+0.089$, $+0.137$] & $+0.233$ [$+0.191$, $+0.267$] & yes \\
cat/truck & 1.80 [1.33, 1.99] & $+0.089$ [$+0.055$, $+0.130$] & $+0.142$ [$+0.112$, $+0.173$] &  \\
deer/dog & 0.24 [0.20, 0.30] & $+0.215$ [$+0.196$, $+0.235$] & $+0.213$ [$+0.175$, $+0.263$] & yes \\
deer/frog & 0.40 [0.36, 0.49] & $+0.040$ [$+0.010$, $+0.065$] & $-0.125$ [$-0.148$, $-0.073$] &  \\
deer/horse & 0.33 [0.27, 0.40] & $+0.170$ [$+0.147$, $+0.193$] & $+0.246$ [$+0.210$, $+0.270$] & yes \\
deer/ship & 2.69 [2.43, 3.63] & $-0.026$ [$-0.049$, $-0.002$] & $-0.068$ [$-0.091$, $-0.045$] &  \\
deer/truck & 2.20 [1.99, 2.43] & $+0.073$ [$+0.046$, $+0.098$] & $+0.092$ [$+0.045$, $+0.118$] &  \\
dog/frog & 0.49 [0.44, 0.54] & $-0.055$ [$-0.079$, $-0.028$] & $-0.138$ [$-0.182$, $-0.119$] &  \\
dog/horse & 0.36 [0.30, 0.44] & $+0.114$ [$+0.089$, $+0.137$] & $+0.038$ [$-0.034$, $+0.059$] &  \\
dog/ship & 2.43 [2.20, 2.97] & $+0.088$ [$+0.065$, $+0.112$] & $+0.076$ [$+0.040$, $+0.106$] &  \\
dog/truck & 1.99 [1.80, 2.20] & $+0.069$ [$+0.043$, $+0.094$] & $-0.092$ [$-0.178$, $-0.028$] &  \\
frog/horse & 0.81 [0.73, 0.89] & $+0.231$ [$+0.211$, $+0.251$] & $+0.210$ [$+0.188$, $+0.232$] & yes \\
frog/ship & 2.69 [2.43, 3.28] & $+0.152$ [$+0.129$, $+0.173$] & $+0.245$ [$+0.223$, $+0.279$] & yes \\
frog/truck & 1.99 [1.80, 2.20] & $+0.175$ [$+0.155$, $+0.198$] & $+0.284$ [$+0.261$, $+0.304$] & yes \\
horse/ship & 2.69 [2.43, 3.28] & $+0.001$ [$-0.029$, $+0.028$] & $-0.109$ [$-0.135$, $-0.020$] &  \\
horse/truck & 1.47 [1.20, 1.63] & $+0.005$ [$-0.025$, $+0.036$] & $+0.040$ [$-0.003$, $+0.234$] &  \\
ship/truck & 1.20 [0.98, 1.47] & $+0.044$ [$+0.009$, $+0.069$] & $-0.020$ [$-0.047$, $+0.009$] &  \\

\bottomrule
\end{tabular}
\end{table}

\begin{table}[h]
\centering
\caption{Pretrained unconditional EDM against the data, per pair (reverse SDE, 14{,}500 trajectories; the reference subsampled to the same size). All values are medians of 100 paired bootstrap resamples, and the interval is the 95\% interval of the ratio. $h$ is raw, before rebalancing.}
\label{tab:pre_full}
\scriptsize
\begin{tabular}{lcccccc}
\toprule
pair & $\sbif$ data & $\sbif$ model & ratio & 95\% interval & $h$ data & $h$ model \\
\midrule
airplane/automobile & 1.39 & 1.15 & 0.82 & [0.68, 1.07] & $+0.007$ & $-0.109$ \\
airplane/bird & 2.15 & 0.91 & 0.47 & [0.24, 0.78] & $+0.018$ & $+0.159$ \\
airplane/cat & 1.79 & 2.56 & 1.36 & [1.13, 1.68] & $-0.087$ & $+0.016$ \\
airplane/deer & 2.88 & 2.15 & 0.71 & [0.44, 1.28] & $-0.013$ & $+0.082$ \\
airplane/dog & 1.90 & 2.28 & 1.20 & [0.97, 1.47] & $-0.061$ & $-0.044$ \\
airplane/frog & 2.72 & 2.15 & 0.79 & [0.58, 1.03] & $-0.232$ & $-0.251$ \\
airplane/horse & 1.39 & 1.31 & 0.91 & [0.73, 1.07] & $-0.072$ & $-0.159$ \\
airplane/ship & 0.50 & 0.58 & 1.16 & [0.93, 1.45] & $-0.086$ & $-0.059$ \\
airplane/truck & 1.07 & 0.94 & 0.88 & [0.72, 1.07] & $-0.033$ & $-0.149$ \\
automobile/bird & 1.90 & 1.31 & 0.73 & [0.61, 0.88] & $-0.040$ & $+0.244$ \\
automobile/cat & 1.48 & 1.58 & 1.06 & [0.83, 1.29] & $-0.143$ & $+0.129$ \\
automobile/deer & 1.68 & 1.68 & 1.00 & [0.83, 1.13] & $-0.136$ & $+0.117$ \\
automobile/dog & 1.90 & 2.02 & 1.06 & [0.89, 1.20] & $-0.093$ & $+0.088$ \\
automobile/frog & 1.79 & 1.90 & 1.06 & [0.88, 1.20] & $-0.164$ & $-0.058$ \\
automobile/horse & 1.58 & 1.68 & 1.00 & [0.88, 1.21] & $-0.075$ & $-0.105$ \\
automobile/ship & 1.15 & 1.07 & 0.94 & [0.64, 1.14] & $-0.079$ & $+0.060$ \\
automobile/truck & 0.43 & 0.58 & 1.25 & [0.86, 1.68] & $-0.081$ & $-0.065$ \\
bird/cat & 0.25 & 0.27 & 1.00 & [0.49, 1.76] & $-0.016$ & $-0.073$ \\
bird/deer & 0.21 & 0.19 & 0.92 & [0.63, 1.28] & $-0.059$ & $-0.056$ \\
bird/dog & 0.37 & 0.37 & 1.00 & [0.73, 1.46] & $+0.044$ & $-0.099$ \\
bird/frog & 0.43 & 0.29 & 0.68 & [0.44, 0.86] & $-0.028$ & $-0.186$ \\
bird/horse & 0.40 & 0.56 & 1.45 & [1.12, 1.88] & $+0.232$ & $-0.081$ \\
bird/ship & 2.42 & 1.58 & 0.69 & [0.48, 1.07] & $-0.020$ & $-0.152$ \\
bird/truck & 2.42 & 1.22 & 0.51 & [0.41, 0.73] & $+0.025$ & $-0.295$ \\
cat/deer & 0.37 & 0.34 & 0.92 & [0.43, 1.58] & $+0.009$ & $+0.053$ \\
cat/frog & 0.37 & 0.40 & 1.04 & [0.64, 1.59] & $+0.053$ & $-0.073$ \\
cat/horse & 0.32 & 0.66 & 2.35 & [1.61, 3.80] & $+0.252$ & $+0.026$ \\
cat/ship & 1.90 & 2.88 & 1.51 & [1.06, 1.87] & $+0.104$ & $+0.002$ \\
cat/truck & 1.48 & 1.39 & 0.94 & [0.69, 1.29] & $+0.089$ & $-0.158$ \\
deer/dog & 0.27 & 0.50 & 1.85 & [1.26, 2.54] & $+0.195$ & $+0.069$ \\
deer/frog & 0.40 & 0.30 & 0.73 & [0.49, 0.96] & $+0.025$ & $-0.160$ \\
deer/horse & 0.37 & 0.21 & 0.61 & [0.31, 1.14] & $+0.152$ & $-0.198$ \\
deer/ship & 2.42 & 1.90 & 0.79 & [0.55, 1.16] & $-0.037$ & $-0.114$ \\
deer/truck & 2.42 & 1.79 & 0.76 & [0.55, 1.10] & $+0.041$ & $-0.195$ \\
dog/frog & 0.46 & 0.66 & 1.44 & [1.16, 1.80] & $-0.054$ & $-0.094$ \\
dog/horse & 0.40 & 0.58 & 1.40 & [0.80, 1.93] & $+0.089$ & $-0.063$ \\
dog/ship & 2.15 & 2.56 & 1.19 & [0.84, 1.43] & $+0.091$ & $+0.081$ \\
dog/truck & 1.90 & 2.28 & 1.20 & [0.89, 1.62] & $+0.072$ & $-0.084$ \\
frog/horse & 0.82 & 0.94 & 1.15 & [0.97, 1.45] & $+0.224$ & $+0.101$ \\
frog/ship & 2.88 & 2.88 & 1.00 & [0.79, 1.19] & $+0.136$ & $+0.159$ \\
frog/truck & 2.15 & 2.72 & 1.27 & [0.94, 1.56] & $+0.163$ & $+0.073$ \\
horse/ship & 2.15 & 2.15 & 0.94 & [0.68, 1.20] & $-0.005$ & $+0.144$ \\
horse/truck & 1.22 & 1.68 & 1.37 & [1.10, 1.77] & $-0.005$ & $+0.001$ \\
ship/truck & 1.31 & 1.15 & 0.87 & [0.64, 1.14] & $+0.023$ & $-0.093$ \\
\bottomrule
\end{tabular}
\end{table}

\end{document}